\PassOptionsToPackage{pdftex}{xcolor}

\documentclass[11pt, a4paper]{gdm_format}
\usepackage[authoryear, sort&compress, round]{natbib}

\usepackage{amsmath,amsfonts,bm}

\def\eqref#1{equation~\ref{#1}}

\def\1{\bm{1}}

\DeclareMathAlphabet{\mathsfit}{\encodingdefault}{\sfdefault}{m}{sl}
\SetMathAlphabet{\mathsfit}{bold}{\encodingdefault}{\sfdefault}{bx}{n}

\usepackage{hyperref}
\hypersetup{
    colorlinks=true,
    citecolor={black!60},
    linkcolor={blue!50!black},
    urlcolor={blue!50!black}
}

\usepackage[utf8]{inputenc}
\usepackage[T1]{fontenc}

\usepackage{url}
\usepackage{booktabs}
\usepackage{array}
\usepackage[export]{adjustbox}
\usepackage{amsfonts}
\usepackage{amssymb}
\usepackage{amsmath}
\usepackage{amsthm}
\newtheorem{theorem}{Theorem}
\newtheorem{proposition}{Proposition}
\newtheorem{definition}{Definition}

\usepackage{nicefrac}
\usepackage{microtype}
\usepackage[table]{xcolor}
\definecolor{rowhl}{RGB}{230,242,255}
\usepackage{graphicx}
\usepackage{algorithm}
\usepackage{algorithmic}
\usepackage{float}
\usepackage{multirow}
\usepackage{pifont}
\usepackage{subcaption}
\usepackage{twemojis}
\usepackage{CJKutf8}
\usepackage{mathtools}
\usepackage{enumitem}

\usepackage{wrapfig}
\usepackage{tabularray} 
\usepackage{tabularx}
\usepackage{placeins}

\usepackage{fontawesome5}
\usepackage{xspace}

\usepackage{eso-pic}
\newcommand{\TitlePageLogo}{%
  \AddToShipoutPictureBG*{%
    \AtPageUpperLeft{%
      \put(\dimexpr 0.75in\relax,-\dimexpr 0.8in\relax){%
        \includegraphics[width=5cm]{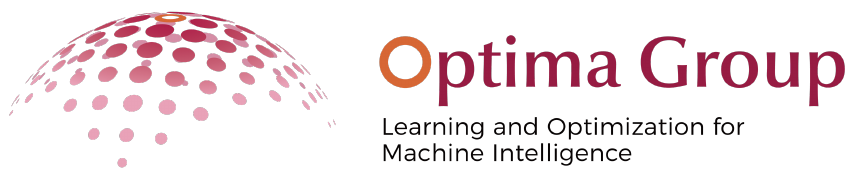}%
      }%
    }%
  }%
}

\newcommand{\ourmethod}{\texttt{NSC}}
\newcommand{\ouralgo}{\texttt{NSC-DP}}    

\usepackage{wasysym}

\titlespacing*{\paragraph}{0pt}{\parskip}{1em}

\title{Neural Spectral Capacity:\\Measuring and Designing Architectures from Network Specification Alone}

\correspondingauthor{Zhichao Lu (zhichao.lu@cityu.edu.hk)}

\author[ ]{Chenyu Zhu}
\author[ ]{Ruoyu Zhao}
\author[ ]{Zhichao Lu}

\affil[ ]{Department of Computer Science, City University of Hong Kong}

\begin{document}

\TitlePageLogo

\begin{abstract}
    Modern Transformer design and compression both reduce to allocating capacity under a budget.
    The standard scalars for these decisions, \#Params and \#FLOPs, capture size and compute
    but not architectural structure---two architectures with identical parameter budgets
    but different depth-width, head, or FFN allocations receive identical scores yet behave differently.
    We propose Neural Spectral Capacity (\ourmethod{}), a closed-form scalar grounded in the
    singular-value spectrum of each weight matrix. Under standard random initialization,
    the Marchenko--Pastur law renders \ourmethod{} computable from the architectural specification alone,
    with no model instantiation, data, or gradients. Its layer-wise additive structure admits \ouralgo{},
    an exact dynamic-programming solver returning the architecture globally maximizing \ourmethod{}
    under resource constraints in seconds on a CPU---a guarantee that black-box search over existing training-free
    proxies cannot provide.
    Empirically, \ourmethod{} outperforms \#Params, \#FLOPs, and representative training-free proxies in ranking across
    seven Transformer and CNN families (on FlexiBERT, $\tau{=}0.505$ on pairs differing in \#Params by ${<}10\%$,
    where \#Params collapses to $0.082$); \ouralgo{} discovers a Transformer-XL architecture on
    WikiText-103 that beats the human-designed baseline in $2$ seconds; and prunes LLaMA-7B to
    the best 5.7B model across eight commonsense reasoning tasks \emph{without any calibration data},
    ${\sim}5900{\times}$ faster than the strongest training-free proxy baseline.
\end{abstract}

\maketitle


\begin{center}
  \href{https://github.com/Optima-CityU/neural-spectral-capacity}{\faGithub  \xspace \texttt{https://github.com/Optima-CityU/neural-spectral-capacity}}
\end{center}

\section{Introduction} \label{sec:introduction}

Designing or compressing a modern Transformer reduces to a single underlying question: under a fixed parameter or compute budget, how should capacity be distributed across architectural components?
At pretraining time, this manifests as choices over depth-width tradeoffs, FFN ratio, attention-head sharing (GQA, MQA~\citep{ainslie2023gqa}), and per-layer expert allocation in MoEs~\citep{jiang2024mixtral,deepseekai2024deepseekv3}.
At post-training time, structured pruning of frontier models~\citep{xia2024sheared,ma2023llmpruner,munoz2024lonas} poses the same question in reverse: under a deployment budget, which components should be retained?
These decisions are currently made by intuition, small-scale ablation, or expensive search.
The standard scalar tools available to a practitioner are \#Params and \#FLOPs, which respectively capture \emph{size} and \emph{compute} but say nothing about \emph{structure}---two architectures with the same parameter budget but different depth-width tradeoffs, head allocations, or FFN ratios receive identical scores yet behave differently when trained or deployed.
Yet structure is precisely what an architectural choice \emph{is}.
Is there a scalar that captures architectural structure, computable from the specification alone?

\paragraph{Why existing training-free proxies fall short for this question.}
The closest existing candidates come from training-free proxies for neural architecture search~\citep{abdelfattah2021zerocost,mellor2021naswot,wang2025wpca}: scalar functions that score architectures at initialization.
Two limitations make them awkward as capacity-allocation tools.
First, each proxy instantiates a randomly-initialized network and \emph{measures} something on it (activation patterns, gradient norms, attention focus), with the resulting score depending on (architecture, input): real data, synthetic samples, or random Gaussian vectors are all valid choices, and the score changes with the choice.
Second, each proxy is a black-box scalar function of the entire network, so optimizing it under resource constraints requires evolutionary or random search, returning the best architecture \emph{visited} rather than the best architecture \emph{possible} under the proxy.
Both limitations are amplified at LLM scale, where instantiating each candidate is expensive and the search space is too large to explore by sampling.

\paragraph{Neural Spectral Capacity.}
We propose \emph{Neural Spectral Capacity} (\ourmethod{}), a scalar quantity that is a function of the architectural specification alone: deterministic in the spec, computable without ever instantiating a model.
Like \#Params and \#FLOPs, \ourmethod{} takes the architectural spec as input and returns a number---no random initialization, no input data, no forward or backward pass, no learned hyperparameters.
The construction is grounded in the singular-value spectrum of each weight matrix, which we connect to the mutual information of the corresponding linear Gaussian channel~\citep{telatar1999capacity}.
Per-matrix capacity is summed across the network with multi-head attention decomposed per head, yielding an architecture-level score that is layer-wise additive.
Under standard random initialization, the Marchenko--Pastur law~\citep{marchenko1967distribution} renders the expected per-matrix capacity exactly computable from the matrix's dimensions and initialization variance alone, so the entire score reduces to a closed-form expression in architectural parameters.

\paragraph{Two practical advantages of \ourmethod{}.}
\ourmethod{}'s structure has two practical consequences that competing proxies do not jointly admit.
First, \ourmethod{} \emph{discriminates} between architectures where \#Params and FLOPs cannot.
At a fixed parameter or compute budget---the typical setting for capacity-allocation decisions in practice---all candidate architectures have nearly identical \#Params and FLOPs by construction, so neither quantity can distinguish them.
\ourmethod{} is sensitive to structural choices---depth, head allocation, FFN ratio---that \#Params and FLOPs do not see, and assigns different scores accordingly.
Second, \ourmethod{}'s layer-wise additivity admits a structural property no other quality-tracking proxy provides: the resource-constrained maximization of \ourmethod{} reduces to a bounded knapsack solvable exactly by dynamic programming.
We exploit this with \ouralgo{} (Algorithm~\ref{alg:nsc_dp}), which returns the architecture \emph{globally maximizing} \ourmethod{} under resource constraints---both globally optimal under the proxy and orders of magnitude faster than heuristic alternatives, on a single CPU core for spaces of up to $10^{32}$ candidate architectures.

The contributions of this work are threefold:
\begin{itemize}[leftmargin=1.5em, itemsep=0.3em, topsep=0.3em]
    \item \textbf{\ourmethod{}, a closed-form architectural scalar for Transformers} that is deterministic in the architectural spec---no instantiation, no input data, no gradient computation.
    Grounded in the singular-value spectrum of each weight matrix, \ourmethod{} reduces under standard random initialization to a deterministic expression in dimensions and initialization variance alone via the Marchenko--Pastur law (Theorem~\ref{thm:mp}, \S\ref{sec:spectral_capacity}--\S\ref{sec:closed_form}).
    Unlike \#Params and \#FLOPs, \ourmethod{} captures architectural structure---depth, head allocation, FFN ratio---that size and compute miss.

    \item \textbf{\ouralgo{}, an exact dynamic-programming solver enabled by \ourmethod{}'s additive structure} (\S\ref{sec:nsc_dp}).
    \ouralgo{} returns the architecture \emph{globally maximizing} \ourmethod{} subject to resource constraints---a guarantee that black-box search (evolutionary, RL, random) over non-decomposable proxies fundamentally cannot provide.
    Among architectural scalars, only \ourmethod{} jointly tracks quality and decomposes additively: \#Params is additive but does not track quality, and existing training-free proxies track quality but are non-decomposable.

    \item \textbf{Empirical validation across ranking, design, and compression.}
    Across seven Transformer and CNN families, \ourmethod{} ranks architectures more accurately than \#Params, \#FLOPs, and training-free proxies in microseconds on CPU---on FlexiBERT, retaining $\tau{=}0.505$ on pairs within $10\%$ of \#Params, where \#Params drops to $0.082$ (\S\ref{sec:exp_proxy}).
    On Transformer-XL, \ouralgo{} beats the human-designed baseline in $2$ seconds, over $400{\times}$ faster than training-free proxies with heuristic search (\S\ref{sec:exp:search}).
    On LLaMA-7B pruning to $5.7$\,B, \ouralgo{} produces the best model across eight commonsense tasks \emph{without calibration data}, ${\sim}5900{\times}$ faster than the strongest baseline (\S\ref{sec:exp:llama}).
\end{itemize}

\section{Related work} \label{sec:related}

\paragraph{Random matrix theory in deep learning.}
Random matrix theory (RMT) has been applied to the spectra neural networks \emph{give rise to}: trained weight matrices~\citep{martin2021implicit}, activation Gram matrices~\citep{pennington2017nonlinear}, input--output Jacobians~\citep{pennington2017resurrecting}, loss Hessians~\citep{pennington2017geometry}, and neural tangent kernels~\citep{fan2020spectra}, alongside precise generalization asymptotics for high-dimensional linear models~\citep{advani2020high,mei2022generalization}; see \citet{couillet2022random} for a survey.
The three works closest to ours span the network lifecycle; Table~\ref{tab:rmt_lifecycle} in Appendix~\ref{app:rmt_table} organizes the comparison.
\citet{martin2021implicit} read heavy-tailed departures of trained-weight spectra from the Marchenko--Pastur (MP) law as a diagnostic of training quality; at initialization, where \ourmethod{} operates, that diagnostic is degenerate by construction.
\citet{pennington2017nonlinear} derive the limiting activation spectrum at initialization, with the nonlinearity and the input distribution both entering the answer; \ourmethod{} abstracts from both.
\citet{berlyand2023enhancing} prune singular values inside the MP bulk during training, reading them as residual initialization randomness; at initialization the entire spectrum lies in the bulk, so their criterion classifies as noise exactly the mass \ourmethod{} counts---noise relative to learned structure, budget relative to what a specification provides for training to use.
\ourmethod{} differs from all three in the role the MP law plays: not a null model, a noise criterion, or a baseline to extend, but a \emph{deterministic equivalent} that renders the score a well-defined function of the specification---the capacity of the MP spectrum itself is what \ourmethod{} counts (\S\ref{sec:closed_form}).
These works characterize the spectra an existing network gives rise to, in order to understand or repair it; \ourmethod{} computes a capacity from a specification, in order to choose among specifications, needing nothing beyond the specification's matrix shapes and initialization variances and turning the spectral quantity into an exactly optimized objective (\S\ref{sec:nsc_dp}).

\paragraph{Training-free proxies and architecture search.}
A line of work scores architectures at initialization to bypass training cost in NAS.
Early proxies originate from pruning at initialization---SNIP~\citep{lee2019snip}, SynFlow~\citep{tanaka2020pruning}---or from activation diversity (NASWOT~\citep{mellor2021naswot}).
Subsequent work targets Transformer structure directly: TF-TAS~\citep{zhou2022tftas}, W-PCA~\citep{wang2025wpca}, ZeroLM~\citep{chen2025zerolm}, AZ-NAS~\citep{lee2024aznas}, and softmax-confidence proxies established alongside the FlexiBERT and GPT-2 benchmarks~\citep{serianni2023nasbench}.
Despite the ``zero-cost'' label, these proxies are not cost-free---each instantiates a randomly-initialized network and measures activation patterns, gradient norms, or attention scores on sampled inputs.
They are typically paired with black-box search---evolutionary algorithms~\citep{real2019regularized}, reinforcement learning, or random sampling---which treats the proxy as a fitness function and returns only the best architecture \emph{visited}.
\ourmethod{} differs on both axes: it is closed-form and data-agnostic (no instantiation, no inputs, no gradients), and its layer-wise additivity admits an exact dynamic-programming solver (\ouralgo{}, \S\ref{sec:nsc_dp}) returning the architecture globally maximizing the proxy objective subject to resource constraints---converting black-box proxy sampling into structured optimization.

\paragraph{Structured pruning of large language models.}
A growing body of work compresses pretrained LLMs along structured axes under a deployment budget.
LLM-Pruner~\citep{ma2023llmpruner} prunes coupled structures via gradient-based importance on a small calibration set; Sheared LLaMA~\citep{xia2024sheared} jointly learns pruning masks and continues pretraining; LoNAS~\citep{munoz2024lonas} expresses pruning as supernet-based architecture selection.
Magnitude- and saliency-based one-shot pruners (SparseGPT~\citep{frantar2023sparsegpt}, Wanda~\citep{sun2024wanda}) require calibration activations to score individual weights.
\ourmethod{} positions in this landscape as a score function that operates from the architectural specification of the pruned subnetwork alone, requiring neither the pretrained weights nor calibration data; we apply it via the LoNAS supernet in \S\ref{sec:exp:llama}.

\section{Neural Spectral Capacity (\ourmethod{})} \label{sec:method}
Our approach builds on the connection between the singular-value spectrum of a weight matrix and the information capacity of the corresponding linear Gaussian channel~\citep{telatar1999capacity}.
We formalize this per-matrix quantity as \emph{spectral capacity} (\S\ref{sec:spectral_capacity}) and show that under standard random initialization it admits a closed-form expression in matrix dimensions and initialization variance alone (\S\ref{sec:closed_form}).
Aggregating per-matrix capacities---within layers, then across them---yields \ourmethod{} (\S\ref{sec:nsc_definition}), whose layer-wise additivity admits the exact dynamic-programming solver \ouralgo{} for resource-constrained architecture search (\S\ref{sec:nsc_dp}).

\subsection{Spectral capacity of a weight matrix} \label{sec:spectral_capacity}
We motivate our per-matrix score by treating each weight matrix as a communication channel and asking how much information can flow through it.
For a weight matrix $W \in \mathbb{R}^{m \times n}$ acting as a linear Gaussian channel $y = Wx + z$ with $x \sim \mathcal{N}(0, I_n)$ and $z \sim \mathcal{N}(0, I_m)$ independent, the mutual information admits a closed-form expression in terms of $W$ alone~\citep{telatar1999capacity}: 
\begin{equation}
\label{eq:mut_info}
I(x; y) = \tfrac{1}{2} \ln\det(I + W^{\!\top} W) = \tfrac{1}{2} \sum_{i} \ln(1 + \sigma_i^2), 
\end{equation}
where $\{\sigma_i\}$ are the singular values of $W$ (proof in Appendix~\ref{app:proof_mi}; we use $\ln$ throughout, measuring information in nats).
The decomposition into a sum over singular values follows from the SVD: in the basis of right singular vectors, the channel splits into $\min(m,n)$ \emph{parallel sub-channels}, the $i$-th contributing $\frac{1}{2}\ln(1 + \sigma_i^2)$.
We take the log-determinant itself as the per-matrix score, dropping the constant factor $\tfrac{1}{2}$: every downstream use of the score---ranking, additive aggregation, exact maximization---is invariant to positive scaling.
\begin{definition}[Spectral Capacity]
\label{def:psi}
\begin{equation}
    \psi(W) \;\coloneqq\; \ln\det\!\big(I + W^{\!\top} W\big) \;=\; \sum_{i=1}^{\min(m,n)} \ln(1 + \sigma_i^2).
    \label{eq:psi}
\end{equation}
\end{definition}
\noindent Zero singular values contribute zero, so $\psi$ remains well-defined for rank-deficient matrices.

\paragraph{What the channel model assumes---and what it does not.}
The Gaussian assumptions attach to the probe signal $x$ and the noise $z$, not to the network's data or weights.
The isotropic input $x \sim \mathcal{N}(0, I_n)$ places unit power in every input direction, so each parallel sub-channel above is probed at signal-to-noise ratio $\sigma_i^2$: this choice is what makes $\psi$ a function of the singular values alone, and is also the optimal input when no realization of $W$ is available~\citep{telatar1999capacity}---our specification-only setting.
The channel serves to define $\psi$; we do not claim that the network's forward pass realizes it.

\paragraph{$\psi$ depends on the full spectrum, not just the parameter count.}
A natural concern is whether $\psi$ is just a re-expression of model size.
It is not.
Two matrices with the same dimensions $(m, n)$---hence the same parameter count $mn$---can have arbitrarily different spectral capacities depending on their singular-value spectrum.
For instance, a full-rank $4096 \times 11\,008$ matrix from a LLaMA-7B FFN layer has $\psi \approx 3174$, while its rank-1 truncation has $\psi \approx 2.5$---a $1\,270\times$ drop with $mn$ unchanged.
\begin{wrapfigure}{r}{0.45\linewidth}
\centering
\vspace{-1.0em}
\includegraphics[width=\linewidth]{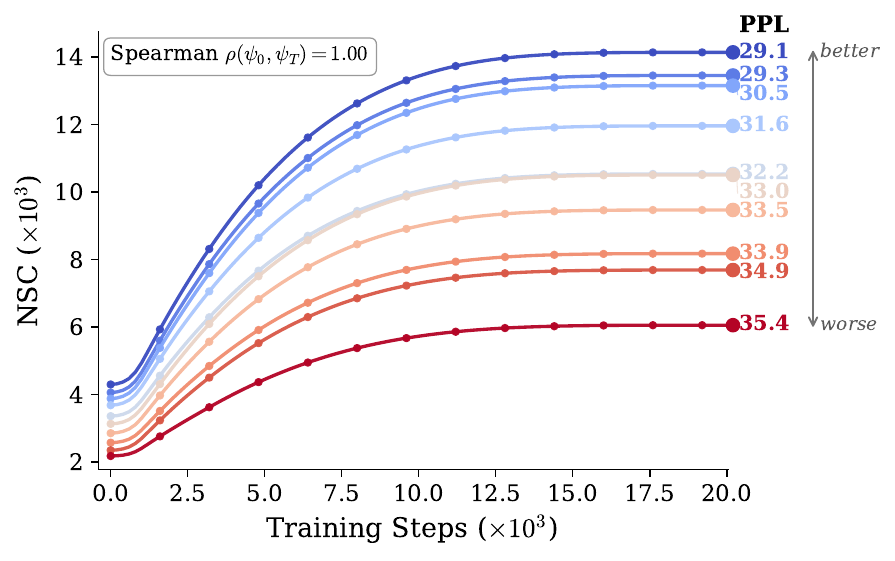}
\caption{Spectral capacity tracks learning. Curves show $\psi_t$ across training for 10 GPT-2 architectures on WikiText-103; color and right-axis label indicate final perplexity. Initialization-time ranking is preserved throughout training, and converged values correlate with PPL.}
\label{fig:motivation_psi}
\vspace{-2.5em}
\end{wrapfigure}
More generally, $\psi$ is sensitive to architectural changes that parameter count cannot see---spectral redistribution under low-rank adapters, pruning, or head merging, and aspect-ratio differences at fixed $mn$ (Appendix~\ref{app:nsc_vs_params}).

\paragraph{Why does $\psi_0$ predict trained performance?}
A second concern is that $\psi$ is computed from randomly initialized weights that training will overwrite.
Let $\psi_t$ denote $\psi$ summed over the network's weight matrices after $t$ training steps, with $t = T$ at convergence. Figure~\ref{fig:motivation_psi} shows that across training: (i)~$\psi_t$ increases monotonically as spectral structure develops; (ii)~rankings at $t=0$ are preserved throughout ($\rho(\psi_0, \psi_T) = 1.00$); (iii)~converged $\psi_T$ ranks final PPL perfectly ($|\rho| = 1.00$).
The perfect correlations reflect the small sample of 10 architectures with deliberate size spread; larger-scale results in §\ref{sec:experiments} report Spearman $\rho \in [0.78, 0.97]$.
A partial mechanistic explanation: $\partial \psi / \partial \sigma_i = 2\sigma_i/(1+\sigma_i^2)$ peaks at $\sigma_i = 1$ (Appendix~\ref{app:proof_gradient}), where variance-preserving initialization places the bulk of the spectrum~\citep{glorot2010understanding, he2015delving}.

\subsection{Spectral capacity in closed form} \label{sec:closed_form}

Definition~\ref{def:psi} requires the singular-value spectrum of $W$, which appears to demand instantiating the matrix and computing an SVD.
We now show that under standard random initialization, $\psi$ admits a deterministic expression depending only on the matrix's dimensions and initialization variance.

\paragraph{The Marchenko--Pastur form of $\psi$.}
By the Marchenko--Pastur theorem~\citep{marchenko1967distribution}, when $W \in \mathbb{R}^{m \times n}$ has i.i.d.\ entries with mean zero and variance $s^2$, the empirical distribution of the normalized squared singular values $\sigma_i^2 / (Ms^2)$ converges, as $\min(m,n) \to \infty$, to a deterministic law with density $f_{\mathrm{MP}}(\lambda; \gamma)$, where $M = \max(m,n)$, $N = \min(m,n)$, $\gamma = N/M$.
This convergence is universal: bounded fourth moment of the entries suffices, covering all standard initialization schemes (Xavier, Kaiming, truncated normal)~\citep{bai2010spectral}.
Integrating $\ln(1 + Ms^2\lambda)$ against this law yields the limiting capacity per sub-channel (proof in Appendix~\ref{app:proof_mp}):
\begin{theorem}[Marchenko--Pastur Spectral Capacity]
\label{thm:mp}
Let $W \in \mathbb{R}^{m \times n}$ have i.i.d.\ entries of mean zero and variance $s^2$, and let $\min(m, n) \to \infty$ with $\gamma$ and $Ms^2$ held fixed. Then
\begin{equation}
\label{eq:mp_limit}
\frac{\psi(W)}{N} \;\xrightarrow{\mathrm{a.s.}}\; \mathcal{I}(\gamma, Ms^2) \;\coloneqq\; \int_{\lambda_-}^{\lambda_+} \ln\!\big(1 + M s^2 \lambda\big) \, f_{\mathrm{MP}}(\lambda; \gamma) \, d\lambda,
\end{equation}
where $\lambda_\pm = (1 \pm \sqrt{\gamma})^2$.
\end{theorem}
$\mathcal{I}(\gamma, Ms^2)$ is the \emph{Shannon transform} of the Marchenko--Pastur law at signal-to-noise ratio $Ms^2$, a standard random-matrix quantity with an explicit closed form~\citep{tulino2004random}.
We henceforth define
\begin{equation}
\label{eq:psi_mp}
\psi(W) \;\coloneqq\; \psi_{\mathrm{MP}}(m, n, s) \;\coloneqq\; N\, \mathcal{I}\big(\gamma,\, M s^2\big).
\end{equation}
The limit functional of Theorem~\ref{thm:mp}, evaluated at the matrix's own aspect ratio and signal-to-noise parameter, is the definition rather than a finite-sample approximation of an SVD-based quantity.
$\psi$ thereby becomes a deterministic function of $(m, n, s)$---of the matrix's place in the architecture specification---with no dependence on a particular weight realization.
At practical dimensions the limit is tight: the relative error of $\psi_{\mathrm{MP}}$ against the SVD-based value $\psi_W$ of Definition~\ref{def:psi} decays as $O(1/\min(m,n))$, falling below $0.4\%$ at $\min(m,n) = 128$---the smallest hidden dimension across our benchmarks---and to ${\sim}10^{-4}$ by $\min(m,n) = 1024$ (Appendix~\ref{app:mp_convergence}).
On the LoNAS-LLaMA-7B trained supernet, $\psi_{\mathrm{MP}}$ matches $\psi_W$ at $\tau{=}\rho{=}1.0000$ across all $129$ Pareto-optimal subnets, with $0.039\%$ mean SVD-oracle regret and a $720{\times}$ compute saving over the SVD route (Appendix~\ref{app:mp_vs_svd}).

\paragraph{Robustness to initialization choice.}
Comparing architectures via Eq.~\ref{eq:psi_mp} requires a single initialization protocol fixed across all candidates, so that score differences reflect architecture rather than initialization.
Under a fixed scheme, $s$ is itself determined by the matrix's shape and role---Xavier sets $s^2 = 2/(m+n)$, Kaiming fan-in $s^2 = 2/n$, truncated normal a constant---so $\psi_{\mathrm{MP}}$ effectively reduces to a function of the dimensions alone.
The choice of fixed scheme broadly preserves the resulting ranking (Appendix~\ref{app:init_robustness}).

\subsection{From single matrices to layers and networks}
\label{sec:nsc_definition}

The spectral capacity $\psi$ scores one weight matrix.
To score an architecture, we fix which matrices are considered, then aggregate---first within a layer, then across layers.

\paragraph{Which matrices are considered.}
We consider the network's linear-projection weight matrices: per-head attention projections, the attention output projection, and FFN matrices.
For multi-head attention with $H$ heads and width $d$, we consider each head's Q/K/V projections separately rather than treating $W_{QKV} \in \mathbb{R}^{3d \times d}$ as one matrix; this mirrors the forward pass (each head computes attention in its own $d_h = d/H$ subspace) and is necessary to discriminate architectures differing in head count.
Convolutions are reshaped from $W \in \mathbb{R}^{c_{\text{out}} \times c_{\text{in}} \times k \times k}$ to $\widetilde{W} \in \mathbb{R}^{c_{\text{out}} \times c_{\text{in}} k^2}$---the form in which the convolution acts as a linear map on im2col patches---before applying $\psi$.
Embeddings (dictated by vocabulary and sequence length), biases, normalization, and nonlinearities are excluded.

\paragraph{From matrices to a layer.}
\begin{wraptable}{r}{0.54\textwidth}
\centering
\footnotesize
\vspace{-1.5em}
\caption{\textbf{Across-layer aggregation ablation.} Spearman $\rho$ against trained performance (GLUE, perplexity, ImageNet Top-1); number of distinct scores in parentheses ($500$/$200$/$1{,}001$ architectures).}
\label{tab:agg_ablation}
\setlength{\tabcolsep}{4pt}
\renewcommand{\arraystretch}{0.95}
\begin{tabular}{l c c c}
\toprule
\textbf{Aggregation} & \textbf{FlexiBERT} & \textbf{GPT-2} & \textbf{AutoFormer-T} \\
\midrule
\rowcolor{rowhl}
$\sum_l \Psi_l$ & 0.884 (297) & 0.968 (200) & 0.810 (208) \\
$\min_l \Psi_l$ & 0.178 (22) & 0.813 (195) & $-0.173$ (3) \\
$\prod_l \Psi_l$ & 0.917 (362) & 0.712 (200) & 0.816 (458) \\
\bottomrule
\end{tabular}
\vspace{-1.2em}
\end{wraptable}
Within a layer, we aggregate by summation: the capacity of layer $l$ is $\Psi_l \coloneqq \sum_{j=1}^{J_l} \psi(W_{l,j})$.
The sum is itself a spectral capacity: because block-diagonal composition adds log-determinants, $\Psi_l = \psi\big(\mathrm{diag}(W_{l,1}, \ldots, W_{l,J_l})\big)$---the capacity of the layer's matrices composed \emph{in parallel}.
Summation is thus the parallel half of the composition rule for channel capacities~\citep{foggo2023on}, extending the parallel sub-channels of \S\ref{sec:spectral_capacity} from directions within one matrix to matrices within one layer.
For the head dimension this parallelism is literal: the heads operate side by side in disjoint subspaces.

\paragraph{From layers to a network.}
Summing layer capacities gives the network-level score:
\begin{definition}[Neural Spectral Capacity]
\label{def:nsc}
For a network $f$ with $L$ layers:
\begin{equation}
\label{eq:nsc}
    \ourmethod(f) \;=\; \sum_{l=1}^{L} \Psi_l \;=\; \sum_{l=1}^{L} \sum_{j=1}^{J_l} \psi_{\mathrm{MP}}(m_{l,j},\, n_{l,j},\, s_{l,j}).
\end{equation}
\end{definition}
The second equality uses the closed form of \S\ref{sec:closed_form}, making \ourmethod{} evaluable in $O(L)$ from the specification alone.
Across layers, parallelism has no literal counterpart: layers execute in sequence, as do the attention and FFN blocks within each layer.
Summing across them is therefore an empirically motivated modelling choice, not a claim about the end-to-end mutual information of the forward pass.
Table~\ref{tab:agg_ablation} supports this choice: summation ranks best on average and no alternative beats it on more than one benchmark, while the bottleneck (min) rule collapses---assigning the $500$ FlexiBERT architectures only $22$ distinct scores and turning anti-correlated on AutoFormer-Tiny; Appendix~\ref{app:aggregation_choice} gives the full protocol.
\begin{algorithm}[t]
\caption{\ouralgo{}: \ourmethod{}-guided Dynamic Programming}
\label{alg:nsc_dp}
\begin{algorithmic}[1]
\REQUIRE Network-level set $\mathcal{G}$; layer-level sets $\{\mathcal{X}_l(g)\}$; resource budget $B$
\ENSURE Architecture $a^*$ maximizing \ourmethod{} under budget $B$
\STATE $\mathrm{best} \leftarrow -\infty$; \; $a^* \leftarrow \textsc{None}$
\FOR{each network-level configuration $g \in \mathcal{G}$}
    \STATE $L \leftarrow L(g)$ \hfill \COMMENT{number of layers determined by $g$}
    \STATE Precompute per-layer capacities $\Psi_l(\cdot\,; g)$ and costs $c_l(\cdot\,; g)$ \hfill \COMMENT{cached $\psi_{\mathrm{MP}}$}
    \STATE Solve the layer-level knapsack via dynamic programming:
    \[
        \hat{\Psi}(g) \;\leftarrow\; \max_{\{x_l \in \mathcal{X}_l(g)\}_{l=1}^{L}} \sum_{l=1}^{L} \Psi_l(x_l; g)
        \;\;\text{s.t.}\;\; \sum_{l=1}^{L} c_l(x_l; g) \leq B,
        \quad \text{with maximizer } \{x_l^*(g)\}_{l=1}^{L}
    \]
    \IF{$\hat{\Psi}(g) > \mathrm{best}$}
        \STATE $a^* \leftarrow (g,\, x_1^*(g), \ldots, x_L^*(g))$; \; $\mathrm{best} \leftarrow \hat{\Psi}(g)$
    \ENDIF
\ENDFOR
\RETURN $a^*$
\end{algorithmic}
\end{algorithm}

\subsection{Architecture design via \ouralgo{}} \label{sec:nsc_dp}
Coupling \ourmethod{} with structured optimization---enabled by its layer-wise additive form---yields \textbf{\ourmethod{}-guided Dynamic Programming (\ouralgo{}, Algorithm~\ref{alg:nsc_dp})}, a solver that returns the architecture globally maximizing \ourmethod{} under a resource budget.
The construction exploits a two-level decomposition.
A Transformer architecture decomposes as $a = (g, x_1, \ldots, x_{L(g)})$, where $g \in \mathcal{G}$ denotes \emph{network-level} decisions (depth, stage layout, embedding size, etc.) and $x_l \in \mathcal{X}_l(g)$ denotes \emph{layer-level} decisions at layer $l$ (heads, FFN width, MLP ratio, etc.).
The network-level $g$ both fixes the number of layers $L(g)$ and enters the layer capacities (e.g., via embedding size).
Conditioned on $g$, \ourmethod{} separates across layers: $\ourmethod{}(a) = \sum_{l=1}^{L(g)} \Psi_l(x_l; g)$, each layer's capacity depending only on its own decisions $x_l$.
Standard deployment constraints admit the same per-layer structure: $R(a) = \sum_{l=1}^{L(g)} c_l(x_l; g) \le B$, where $R$ can represent parameter count, FLOPs, or any resource that decomposes as a sum of per-layer costs.
Once $g$ is fixed, the layer-level problem is therefore a knapsack---bounded or multiple-choice, depending on the search space---solvable exactly by dynamic programming.
Algorithm~\ref{alg:nsc_dp} enumerates $g \in \mathcal{G}$ and solves this inner knapsack for each $g$.
This exact-optimization structure gives \ouralgo{} four properties:
\begin{itemize}[leftmargin=1.5em, itemsep=0.3em, topsep=0.3em]
    \item \textbf{Global optimality.} Because the outer loop exhausts $\mathcal{G}$ and the inner DP solves the layer-level knapsack exactly, \ouralgo{} returns the architecture that \emph{globally maximizes \ourmethod{}} subject to the resource budget---global with respect to the proxy objective, not the underlying architecture design problem. By contrast, existing training-free proxies serve as black-box fitness functions and return only the best architecture visited by heuristic search.

    \item \textbf{Multi-budget search.} The DP table additionally encodes the proxy-optimum at every reachable budget $b \le B$, so a single \ouralgo{} run yields the entire Pareto front over the discretized budget grid as $\mathcal{O}(B)$ table look-ups---a property absent from training-free proxy baselines, which require a separate single-objective search per budget (Appendix~\ref{app:multiobj}).

    \item \textbf{Efficiency.} \ourmethod{} is closed-form and admits caching: each $\psi_{\mathrm{MP}}$ evaluation is a 1D numerical quadrature costing ${\sim}10\,\mu\text{s}$, reused across network-level configurations whenever the $(m, n, s)$ tuple repeats. The inner DP runs in $\mathcal{O}(L(g) \cdot B \cdot |\mathcal{X}_l(g)|)$ on a discretized budget axis---negligible compared to a single forward pass. Combined, \ouralgo{} searches spaces of up to $10^{32}$ architectures in seconds on a single CPU core, with no GPU, no data, and no model instantiation.

    \item \textbf{Generality.} \ouralgo{} applies to any search space where (i)~the \ourmethod{} score decomposes as a sum of independent per-layer contributions, and (ii)~the resource constraint is a sum of per-layer costs. These conditions hold broadly: in AutoFormer, the layer-level variables are MLP ratio and head count; in LoNAS-LLaMA, FFN width and LoRA rank.
\end{itemize}

\section{Experiments} \label{sec:experiments}
We evaluate \ourmethod{} in three settings: (i)~ranking quality across seven Transformer and CNN architecture families (\S\ref{sec:exp_proxy}), (ii)~end-to-end neural architecture search on Transformer-XL and AutoFormer (\S\ref{sec:exp:search}), and (iii)~structured pruning of LLaMA-7B (\S\ref{sec:exp:llama}).
Table~\ref{tab:search-spaces} summarizes the architecture families used across the three settings.

\begin{table}[ht]
\centering
\small
\caption{Architecture families and decision dimensions evaluated in our experiments. \emph{Decision Dimensions} lists the architectural axes that vary across candidate networks within each family; \emph{\#Archs.} denotes the (estimated) total number of distinct architectures.}
\label{tab:search-spaces}
\setlength{\tabcolsep}{4pt}
\begin{tabularx}{\linewidth}{@{}c l X r@{}}
\toprule
 & \textbf{Architecture} & \textbf{Decision Dimensions} & \textbf{\#Archs.} \\
\midrule
\multirow{5}{*}{\rotatebox{90}{\textit{Transformers}}}
 & BERT~\citep{serianni2023nasbench}
    & Hidden dim, depth, per-layer FFN dim \& attn.\ type
    & $1.1 \times 10^{7}$ \\
 & GPT-2~\citep{javaheripi2022litetransformersearch}
    & Embed dim, depth, per-layer FFN dim \& heads
    & ${>}\,10^{54}$ \\
 & Transformer-XL~\citep{javaheripi2022litetransformersearch}
    & Embed dim, depth \& per-layer FFN dim
    & $2 \times 10^{25}$ \\
 & ViT (AutoFormer)~\citep{chen2021autoformer}
    & Embed dim, depth, per-layer heads \& MLP ratio
    & $1.7 \times 10^{16}$ \\
 & LLaMA (LoNAS)~\citep{munoz2024lonas}
    & Per-layer FFN dim \& LoRA rank (32 layers)
    & $1.0 \times 10^{32}$ \\
\midrule
\multirow{2}{*}{\rotatebox{90}{\textit{CNNs}}}
 & CNN (NATS-Bench)~\citep{dong2021nats}
    & Per-layer channels (5 layers)
    & 32{,}768 \\
 & MobileNetV3 (OFA)~\citep{cai2020once}
    & Per-stage depth, per-layer channels \& kernel size
    & $2.0 \times 10^{19}$ \\
\bottomrule
\end{tabularx}
\end{table}

\subsection{Ranking quality across architecture families via \ourmethod{}}
\label{sec:exp_proxy}

Before applying \ourmethod{} to architecture search (\S\ref{sec:exp:search}) and LLM pruning (\S\ref{sec:exp:llama}), we first establish that it tracks trained performance across architecture families.
This subsection asks: does \ourmethod{}, computed from the architectural specification alone, rank candidate architectures more accurately than \#Params, \#FLOPs, and representative training-free proxies that require model instantiation, forward or backward passes, and sampled inputs?

\begin{figure}[t]
    \centering
    \begin{minipage}[c]{0.42\textwidth}
        \centering
        \begin{subfigure}{\textwidth}
            \centering
            \includegraphics[width=\textwidth]{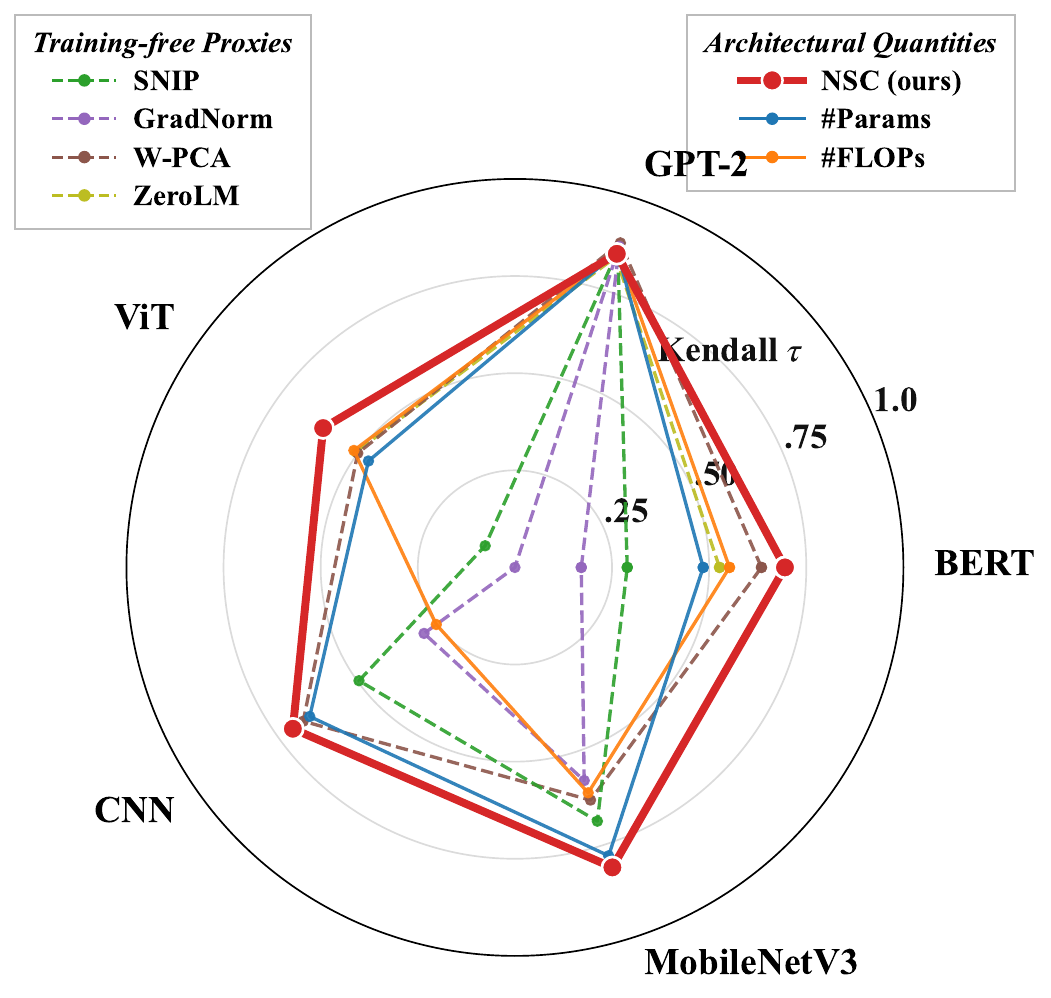}
            \label{fig:radar_overall}
        \end{subfigure}
    \end{minipage}
    \hfill
    \begin{minipage}[c]{0.56\textwidth}
        \centering
        \begin{subfigure}{0.31\textwidth}
            \includegraphics[width=\textwidth]{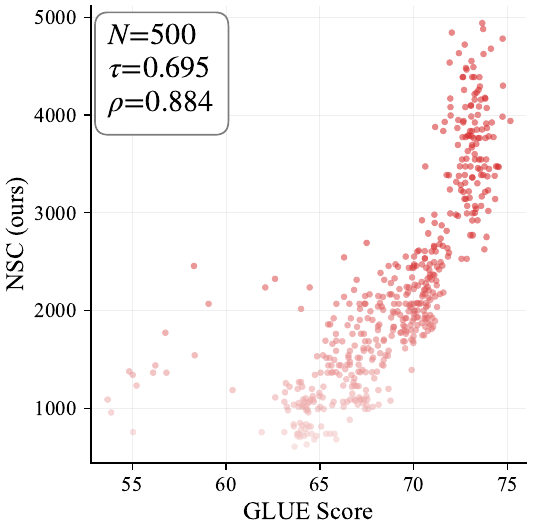}
            \caption{\ourmethod{}}
        \end{subfigure}
        \hfill
        \begin{subfigure}{0.31\textwidth}
            \includegraphics[width=\textwidth]{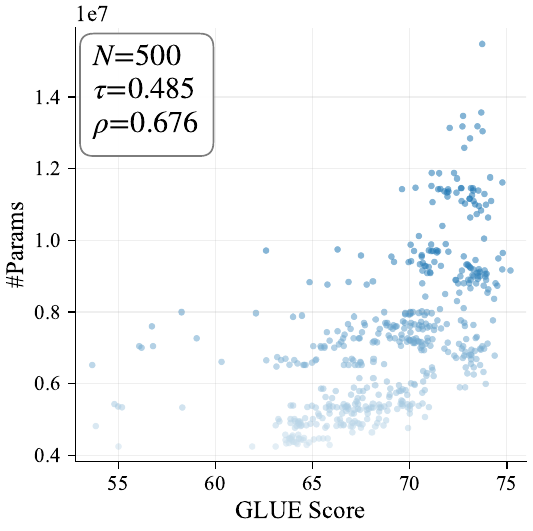}
            \caption{\#Params}
        \end{subfigure}
        \hfill
        \begin{subfigure}{0.31\textwidth}
            \includegraphics[width=\textwidth]{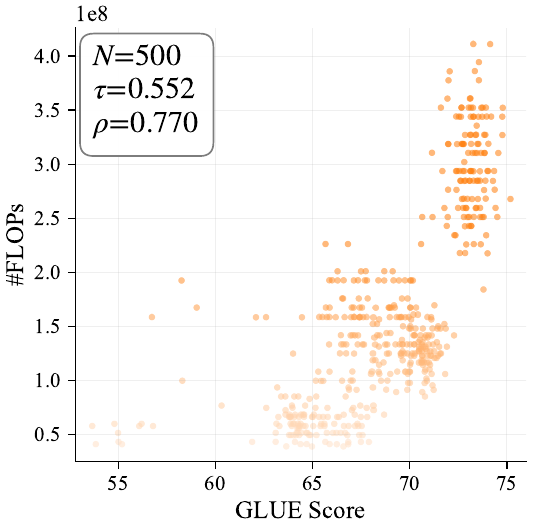}
            \caption{\#FLOPs}
        \end{subfigure}

        \vspace{1mm}
        \begin{subfigure}{0.31\textwidth}
            \includegraphics[width=\textwidth]{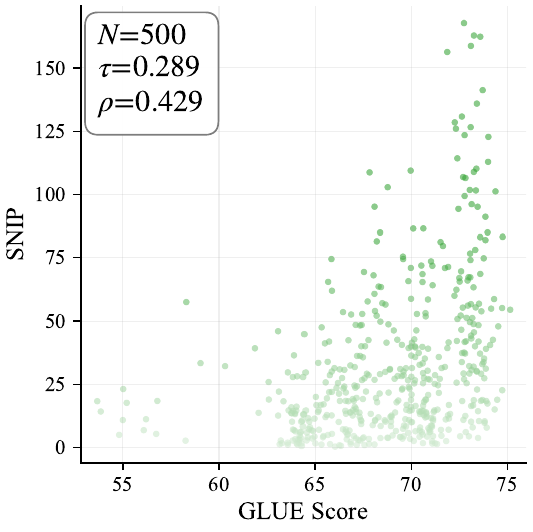}
            \caption{SNIP}
        \end{subfigure}
        \hfill
        \begin{subfigure}{0.31\textwidth}
            \includegraphics[width=\textwidth]{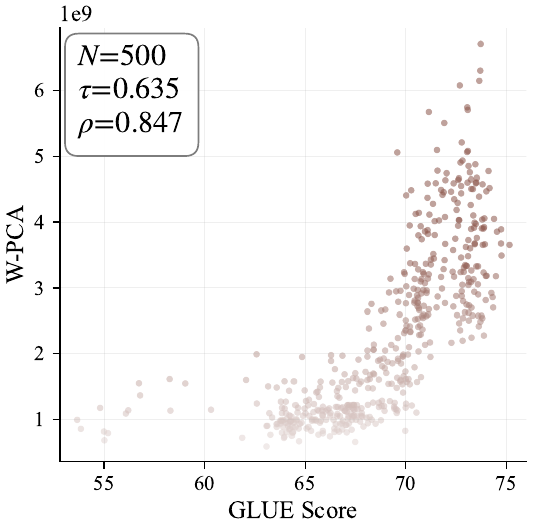}
            \caption{W-PCA}
        \end{subfigure}
        \hfill
        \begin{subfigure}{0.31\textwidth}
            \includegraphics[width=\textwidth]{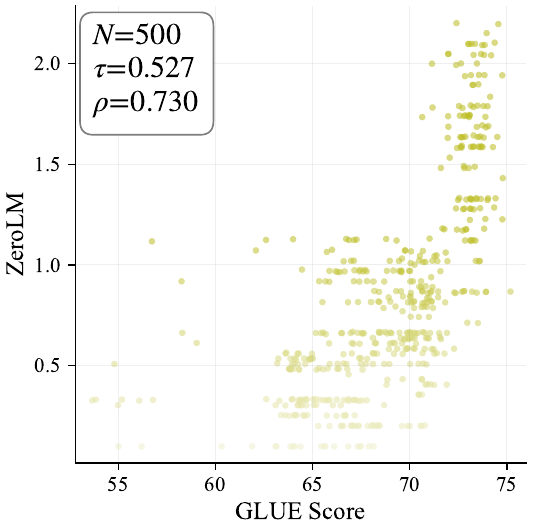}
            \caption{ZeroLM}
        \end{subfigure}
    \end{minipage}

    \caption{Ranking quality of \ourmethod{} against architectural quantities (\#Params, \#FLOPs) and training-free proxies across five architecture families. \textit{Left:} Kendall $\tau$ on FlexiBERT, GPT-2, AutoFormer-Tiny, NATS-Bench-SSS, and MobileNetV3. \textit{Right:} Scatter of six scoring functions against GLUE score on 500 BERT architectures from FlexiBERT, with Kendall $\tau$ and Spearman $\rho$ annotated.}
    \label{fig:radar_bert_scatter}
\end{figure}

\paragraph{Setup.}
We evaluate \ourmethod{} on three Transformer architecture families: 500 BERT architectures from FlexiBERT~\citep{tuli2022flexibert,serianni2023nasbench}, 200 GPT-2 architectures from LiteTransformerSearch~\citep{javaheripi2022litetransformersearch}, and 1{,}001 ViT architectures from the AutoFormer-Tiny supernet~\citep{chen2021autoformer}.
For FlexiBERT, we use the full decision space (dynamic hidden size and FFN), differing from the constrained setting in \citet{serianni2023nasbench} (fixed hidden size and FFN); results under the original setting are reported in Appendix~\ref{app:realdata}.
To assess whether \ourmethod{} generalizes beyond Transformers, we additionally include two CNN architecture families: CNN architectures from NATS-Bench-SSS~\citep{dong2021nats} and MobileNetV3 architectures from the OFA supernet~\citep{cai2020once}.
We compare \ourmethod{} against the two existing architectural quantities (\#Params, \#FLOPs) and four representative training-free proxies (W-PCA~\citep{wang2025wpca}, ZeroLM~\citep{chen2025zerolm}, SNIP~\citep{lee2019snip}, GradNorm~\citep{abdelfattah2021zerocost}); benchmark-specific additional baselines are evaluated in Appendix~\ref{app:realdata}.
Wall-clock times are measured on a single A100 GPU for training-free proxies and a single CPU core for architectural quantities including \ourmethod{}.

\paragraph{\ourmethod{} captures architectural quality beyond size and compute.}
On FlexiBERT, \ourmethod{} reaches $\tau=0.695$ versus $0.485$ for \#Params and $0.552$ for \#FLOPs (\autoref{tab:flexibert-correlation}); under the 10\%-PW control---
\begin{wraptable}{r}{0.42\textwidth}
    \centering
\footnotesize
\vspace{-.2em}
\caption{%
  \textbf{Ranking correlation on FlexiBERT.}
  Methods grouped as in \autoref{fig:radar_bert_scatter}.
  \emph{10\%-PW}: $\tau$ on architecture pairs differing in $\#$Params by ${<}10\%$.
  \emph{Time}: wall-clock for all 500 architectures, on CPU for architectural quantities (top) and on a single A100 GPU for training-free proxies (bottom).
  \textbf{Bold}: best; \underline{underline}: second best.%
}
\label{tab:flexibert-correlation}
\setlength{\tabcolsep}{4pt}
\renewcommand{\arraystretch}{0.95}
\begin{tabular}{l c c r}
\toprule
\textbf{Method} & $\tau$ & \textbf{10\%-PW} $\tau$ & \textbf{Time} \\
\midrule
\rowcolor{rowhl}
\textbf{\ourmethod{} (ours)} & \textbf{0.695} & \textbf{0.505} & \textbf{2\,ms} \\
\#Params               & 0.485          & 0.082          & ---            \\
\#FLOPs                & 0.552          & 0.329            & 2\,ms          \\
\midrule
W-PCA                  & \underline{0.635} & \underline{0.417} & 88\,s \\
ZeroLM                 & 0.527             & 0.355             & 63\,s \\
SNIP                   & 0.289             & 0.237             & 73\,s \\
GradNorm               & 0.171             & 0.164             & 74\,s \\
\bottomrule
\end{tabular}
\vspace{-1em}
\end{wraptable}
which restricts $\tau$ to architecture pairs differing in \#Params by less than $10\%$, removing size as a discriminating signal---\#Params drops to $\tau=0.082$ while \ourmethod{} retains $\tau=0.505$, suggesting that \ourmethod{} captures architectural information beyond parameter count.
This pattern generalizes: \ourmethod{}'s 10\%-PW $\tau$ exceeds \#Params' on every benchmark---e.g., $0.503$ vs.\ $0.322$ on AutoFormer-Tiny, $0.372$ vs.\ $0.116$ on NATS-Bench-SSS, $0.535$ vs.\ $0.414$ on MobileNetV3 (Appendix~\ref{app:controlled-correlation}).
\ourmethod{} additionally outperforms training-free proxies that require model instantiation, forward or backward passes on sampled inputs, or per-dataset proxy search, despite being computed from the architecture specification alone (\autoref{fig:radar_bert_scatter}; full per-family results in Appendix~\ref{app:realdata}).
On FlexiBERT it leads W-PCA by $+0.060$ and ZeroLM by $+0.168$ in $\tau$; the lead widens on AutoFormer-Tiny ($+0.110$ over W-PCA), while on GPT-2 \ourmethod{} is competitive with the leading proxies (within $0.03$ of W-PCA) at orders-of-magnitude lower wall-clock.
\ourmethod{} also top-ranks both CNN architecture families (NATS-Bench-SSS, MobileNetV3), despite being designed for Transformers.
All of these results are obtained at microsecond cost per architecture (\autoref{tab:flexibert-correlation}, \emph{Time} column), four orders of magnitude faster than the closest training-free competitors.

\subsection{Architecture search via \ouralgo{}}
\label{sec:exp:search}

The ranking results above establish that \ourmethod{} tracks trained performance broadly.
We now ask whether \emph{optimizing} \ourmethod{} produces good architectures---first in the classical NAS setting (this section), then via structured pruning of a pretrained LLM (\S\ref{sec:exp:llama}).

\paragraph{Setup.}
We instantiate \ouralgo{} on two Transformer families: Transformer-XL (TXL)~\citep{dai2019transformer} from LiteTransformerSearch~\citep{javaheripi2022litetransformersearch}, and vision Transformers from AutoFormer~\citep{chen2021autoformer} (Table~\ref{tab:search-spaces}).
For each benchmark, all baselines and \ouralgo{} are evaluated under an identical protocol: for TXL, every selected architecture is trained from scratch on WikiText-103 under the same hyperparameters; for AutoFormer-Tiny, every selected architecture inherits weights from the AutoFormer supernet.
The TXL search uses a non-embedding parameter budget of $38.4$\,M $\pm 5\%$ matching TXL Base; AutoFormer-Tiny uses a parameter budget of $5.7$\,M.

\paragraph{\ouralgo{} instantiation and baselines.}
For Transformer-XL, we instantiate Algorithm~\ref{alg:nsc_dp} with network-level decisions $g = (d_{\mathrm{model}}, L)$ and layer-level decisions $x_l = d_{\mathrm{ff}, l}$; the inner problem is a bounded knapsack with additive value $\Psi_l = 2\,\psi_{\mathrm{MP}}(d_{\mathrm{ff}, l}, d_{\mathrm{model}}, s)$ (the FFN pair; attention terms are fixed by $g$) and additive cost $c_l = d_{\mathrm{ff}, l}(2 d_{\mathrm{model}}{+}1)$, with the resource budget $B$ matching the parameter range above. For AutoFormer-Tiny, the network-level decisions are network depth and embedding dimension, with layer-level decisions $x_l = (\text{MLP ratio}_l, \text{heads}_l)$; the inner problem is a multiple-choice knapsack solved analogously by dynamic programming. Baselines reflect each benchmark's prior literature: for Transformer-XL, we compare against the original Transformer-XL base and three training-free proxies---Synaptic Diversity~\citep{zhou2022tftas}, Softmax Confidence~\citep{serianni2023nasbench}, and W-PCA~\citep{wang2025wpca}---each used as fitness for evolutionary search; for AutoFormer-Tiny we compare against the AutoFormer-Tiny oracle~\citep{chen2021autoformer}, TF-TAS~\citep{zhou2022tftas}, AZ-NAS~\citep{lee2024aznas}, and W-PCA proxy used as fitness for evolutionary search. All compared baselines run on GPU while \ouralgo{} runs on a single CPU core.

\paragraph{Results.}
Table~\ref{tab:nas-results} reports both experiments.
On Transformer-XL, \ouralgo{} discovers an architecture with test PPL $23.087$, the lowest among all methods---improving over Transformer-XL Base ($23.279$) and Synaptic Div. ($23.135$). The search completes in $2.0$ seconds on CPU, more than $400\times$ faster than the $883$--$965$\,s consumed by training-free proxy baselines.
On AutoFormer-Tiny, \ouralgo{} matches the day-scale AutoFormer oracle within $0.03$ accuracy points while reducing search cost by six orders of magnitude, and outperforms all heuristics-based alternatives on Acc@1.
Results on AutoFormer-Small/Base are reported in Appendix~\ref{app:autoformer-search} and follow the same pattern; Appendix~\ref{app:discovered-arch} analyzes structural properties of \ouralgo{}'s discovered architectures, including a proof that uniform per-layer FFN allocation is the proxy-optimum on Transformer-XL.

\begin{table}[ht]
\centering
\caption{End-to-end architecture search by \ouralgo{} on language (\textit{left}: Transformer-XL on WikiText-103) and vision (\textit{right}: AutoFormer-Tiny on ImageNet-1K) Transformers. \#P: non-embedding parameters in millions. Search-cost: s = seconds, d = GPU-days. \textbf{Bold}: best; \underline{underline}: second best.}
\label{tab:nas-results}
\setlength{\tabcolsep}{4pt}
\footnotesize

\begin{minipage}[t]{0.41\linewidth}
\centering
\begin{tabular}{l c c c}
\toprule
\textbf{Method} & \textbf{\#P (M)} & \textbf{Test PPL} ($\downarrow$) & \textbf{Cost} \\
\midrule
TXL Base            & 38 & 23.279 & --- \\
Synaptic Div.       & 37 & \underline{23.135} & $903$\,s \\
Softmax Conf.       & 39 & 23.532 & $883$\,s \\
W-PCA               & 40 & 23.669 & $965$\,s \\
\midrule
\rowcolor{rowhl}
\textbf{\ouralgo{} (ours)} & {40} & \textbf{23.087} & $\mathbf{2.0}$\,\textbf{s} \\
\bottomrule
\end{tabular}
\end{minipage}%
\hfill
\begin{minipage}[t]{0.56\linewidth}
\centering
\begin{tabular}{l c c c c}
\toprule
\textbf{Method} & \textbf{\#P (M)} & \textbf{Acc@1} ($\uparrow$) & \textbf{Acc@5} ($\uparrow$) & \textbf{Cost} \\
\midrule
AutoFormer-T     & 5.7 & \textbf{75.308} & 92.690          & $24$\,d \\
TF-TAS              & 5.9 & 75.234          & \underline{92.730} & $0.5$\,d \\
AZ-NAS              & 5.9 & 74.804          & 92.504          & $2{,}592$\,s \\
W-PCA               & 5.8 & 74.752          & 92.500          & $206$\,s \\
\midrule
\rowcolor{rowhl}
\textbf{\ouralgo{} (ours)} & {5.8} & \underline{75.276} & \textbf{92.788} & $\mathbf{0.03}$\,\textbf{s} \\
\bottomrule
\end{tabular}
\end{minipage}

\end{table}

\subsection{Structured pruning of LLaMA-7B via \ouralgo{}}
\label{sec:exp:llama}

We now apply \ouralgo{} to LLM compression: structured pruning of a pretrained LLM under a deployment budget.
This setting is qualitatively distinct from the end-to-end NAS of \S\ref{sec:exp:search}: rather than designing a network from scratch, we compress a pretrained LLM by selecting which components to retain.
The training-free, calibration-data-free property of \ourmethod{} is materially differentiating here---existing structured-pruning methods couple the pruning decision to the pretrained weights and a calibration mini-batch, while \ourmethod{} scores pruning configurations from the specification of the pruned subnetwork alone.

\paragraph{Setup.}
We evaluate \ouralgo{} on structured pruning of LLaMA-7B via the LoNAS SuperNet of \citet{munoz2024lonas}, with parameter budget 5.7\,B.
The SuperNet exposes two elastic axes per Transformer block: per-block LoRA rank $r_\ell \in \{32, 28\}$ applied to attention and FFN projections, and per-block FFN intermediate dimension $h_\ell \in \{11008, 9632, 8256, 6880, 5504\}$.

\paragraph{\ouralgo{} instantiation and baselines.}
The pruning decision space contains only layer-level decisions: each block selects a tuple $x_\ell = (r_\ell, h_\ell)$ of LoRA rank and FFN intermediate dimension.
With network-level depth and embedding dimension fixed by the LLaMA-7B architecture, \ouralgo{} reduces to a single multiple-choice knapsack solved by exact dynamic programming over the closed-form $\psi_{\mathrm{MP}}$ cache.
We compare against four training-free proxies---SNIP~\citep{lee2019snip}, GradNorm~\citep{abdelfattah2021zerocost}, SynFlow~\citep{tanaka2020pruning}, and W-PCA~\citep{wang2025wpca}---each used as the fitness function for an evolutionary search under the same 5.7\,B budget (population size $50$ for $20$ generations, yielding $1{,}000$ proxy evaluations total per method).
Each proxy is computed on a calibration mini-batch from WikiText-103 ($B{=}2$ tokens of length $T{=}2048$); \ouralgo{} requires no calibration data.
The LoNAS-SuperNet (the un-pruned $6.7$\,B configuration) serves as the upper-bound reference.
All baselines run on a GPU; \ouralgo{} runs on a single CPU thread.
We evaluate pruned networks under a unified zero-shot multiple-choice protocol following \citet{munoz2024lonas} on eight commonsense reasoning tasks (BoolQ, PIQA, SIQA, HellaSwag, WinoGrande, ARC-Easy, ARC-Challenge, OpenBookQA) and report the unweighted average $\textsc{Avg}_8$.

\begin{table}[ht]
\centering
\caption{Pruning LLaMA-7B from $6.7$\,B to a $5.7$\,B target via the LoNAS SuperNet: \ouralgo{} versus training-free proxy baselines. The LoNAS-SuperNet row reports the un-pruned performance for reference. \emph{Cost}: search-phase wall-clock (s = seconds, min = minutes). \textbf{Bold}: best per column.}
\label{tab:lonas}
\setlength{\tabcolsep}{3.5pt}
\renewcommand{\arraystretch}{1.05}
\footnotesize
\resizebox{\textwidth}{!}{%
\begin{tabular}{l c c c c c c c c c c r}
\toprule
\textbf{Method} & \textbf{Params} &
BoolQ & PIQA & SIQA & HSwag & WGrnd & ARC-e & ARC-c & OBQA &
\textbf{Avg$_8$} & \textbf{Cost} \\
\midrule
LoNAS-SuperNet & 6.7\,B & 66.02 & 77.58 & 72.93 & 57.49 & 66.85 & 78.41 & 62.29 & 76.20 & 69.72 & --- \\
\midrule
SNIP            & 5.4\,B & 61.93 & 69.21 & 64.99 & 41.26 & 62.04 & 66.50 & 52.47 & 67.40 & 60.73 & 22\,min \\
SynFlow         & 5.2\,B & 61.13 & 70.24 & 63.25 & 45.94 & 59.98 & 66.84 & 51.88 & 66.20 & 60.68 & 23\,min \\
GradNorm        & 5.5\,B & 61.31 & 70.78 & 67.35 & 46.33 & 62.59 & 69.02 & 55.03 & 67.20 & 62.45 & 22\,min \\
W-PCA           & 5.7\,B & 63.82 & 70.95 & 67.91 & 49.24 & 62.75 & 69.57 & 56.40 & 70.00 & 63.83 & 45\,min \\
\midrule
\rowcolor{rowhl}
\textbf{\ouralgo{} (ours)} & \textbf{5.7\,B}
                & \textbf{64.01} & \textbf{72.85} & \textbf{68.94} & \textbf{51.30}
                & \textbf{64.72} & \textbf{72.43} & \textbf{58.70} & \textbf{70.80}
                & \textbf{65.47} & $\mathbf{0.5}$\,\textbf{s} \\
\bottomrule
\end{tabular}%
}
\end{table}

\begin{figure}[t]
\centering
\includegraphics[width=.9\linewidth]{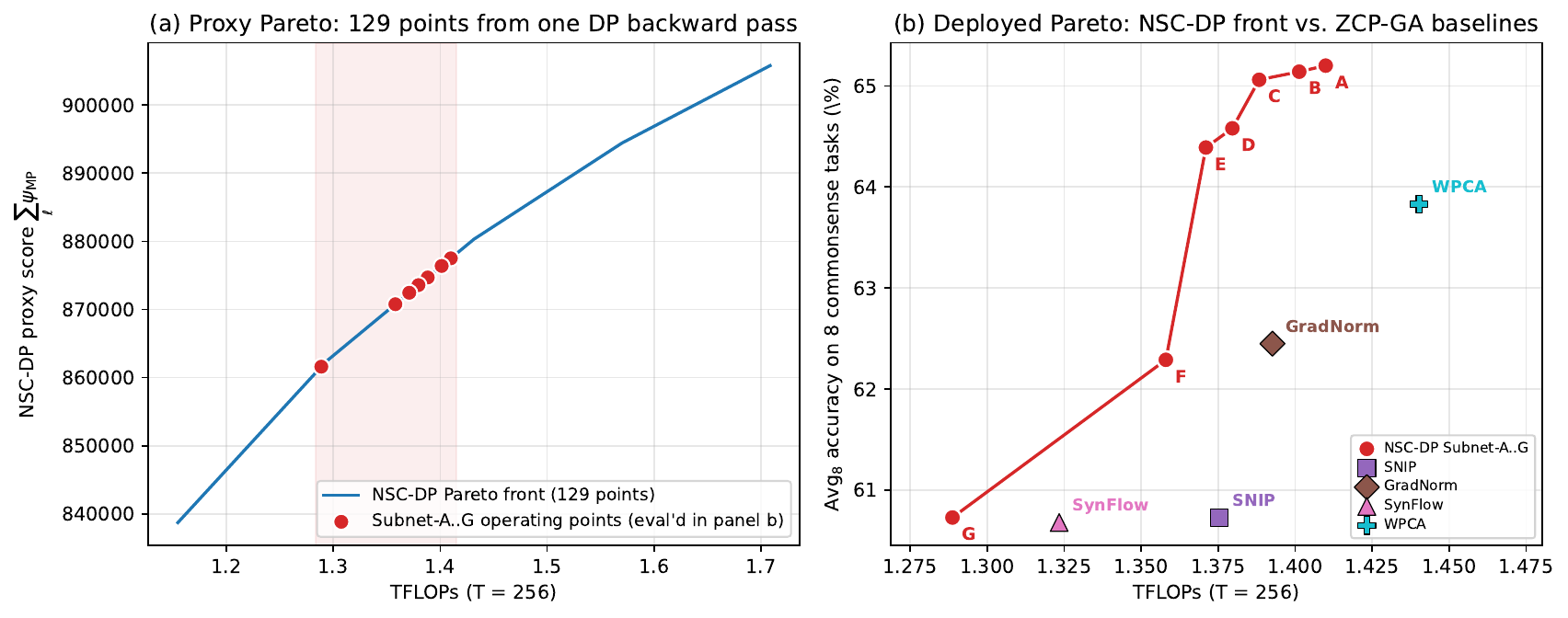}
\caption{\textbf{Pareto front from a single \ouralgo{} run.} Avg$_8$ commonsense reasoning vs.\ TFLOPs (at $T{=}256$) for the seven LoNAS-LLaMA-7B operating points (A--G), all read from the same $0.46$\,s DP backward pass. Full numerical results in Appendix~\ref{app:multiobj}.}
\label{fig:pareto_main}
\end{figure}

\paragraph{Results.}
Table~\ref{tab:lonas} reports four findings.
\paragraph{(1) Best pruned model at matched parameter budget.} \ouralgo{} produces the highest-quality pruned model at the $5.7$\,B budget, outperforming W-PCA by $+1.64$\,pp on $\textsc{Avg}_8$ and dominating every baseline on every individual task.
\paragraph{(2) Specification-based selection at LLM scale.} The four baselines rely on the LoNAS-LLaMA-7B weights and a calibration mini-batch; \ouralgo{} selects from the architectural specification alone, requiring neither.
\paragraph{(3) Search efficiency.} \ouralgo{} completes pruning in 0.46\,s on a CPU thread, $\sim$2{,}800$\times$ faster than the fastest baseline (GradNorm, 22\,min) and $\sim$5{,}900$\times$ faster than the strongest baseline (W-PCA, 45\,min).
\paragraph{(4) Pareto-front output at zero marginal cost.} The same DP backward pass that produces the $5.7$\,B result encodes the optimum at every reachable budget; reading off the seven LoNAS-LLaMA operating points spanning $1.29$--$1.41$\,TFLOPs requires no additional search (Figure~\ref{fig:pareto_main}; Appendix~\ref{app:multiobj}).

\section{Conclusion}
\label{sec:conclusion}

We introduced Neural Spectral Capacity (\ourmethod{}), a closed-form architectural scalar computable from a Transformer's specification alone, and \ouralgo{}, an exact dynamic-programming solver returning the architecture globally maximizing \ourmethod{} under resource constraints. Empirically, \ourmethod{} ranks architectures more accurately than \#Params, \#FLOPs, and representative training-free proxies across seven Transformer and CNN families; \ouralgo{} discovers a Transformer-XL architecture beating the human-designed baseline in $2$ seconds on a single CPU core, and produces the best pruned LLaMA-7B at $5.7$\,B across eight commonsense reasoning tasks \emph{without any calibration data}, ${\sim}5900{\times}$ faster than the strongest baseline. Two further architectural axes extend by construction: per-head decomposition (\S\ref{sec:nsc_definition}) scores MHA/GQA/MQA grouping, and the layer-wise additive sum scores per-layer expert allocation in MoEs. We see \ourmethod{} growing into a general quantitative tool for resource-constrained Transformer architecture choice.


{\small
\bibliographystyle{plainnat}
\bibliography{references}
}

\clearpage
\appendix
\section{RMT in deep learning: lifecycle comparison}
\label{app:rmt_table}

Table~\ref{tab:rmt_lifecycle} organizes the comparison of \S\ref{sec:related} along four axes: when each method applies in the network lifecycle, what it needs beyond the architecture, how it uses the Marchenko--Pastur law, and what it produces.

\begin{table}[h]
\centering
\footnotesize
\caption{\textbf{Random matrix theory in deep learning: what each method needs and produces.} Rows ordered by where in the network lifecycle each method applies.}
\label{tab:rmt_lifecycle}
\setlength{\tabcolsep}{4pt}
\renewcommand{\arraystretch}{1.05}
\begin{tabularx}{\linewidth}{@{}l X X X X@{}}
\toprule
& \textbf{When it applies} & \textbf{Needs beyond the architecture} & \textbf{How it uses the MP law} & \textbf{What it produces} \\
\midrule
\rowcolor{rowhl}
\textbf{\ourmethod{} (ours)} & before any network exists & nothing---matrix shapes and initialization variances suffice & as the quantity itself: the MP spectrum is the capacity being counted & closed-form score + exact argmax (\ouralgo{}) \\
\citet{pennington2017nonlinear} & at initialization & nonlinearity + input distribution & the linear baseline their nonlinear theory extends & limiting spectral density of activations \\
\citet{berlyand2023enhancing} & during training & the partially trained weights & noise criterion: MP-bulk singular values are pruned & a pruned network \\
\citet{martin2021implicit} & after training & the trained checkpoint & the null model that training departs from & a training-quality diagnostic \\
\bottomrule
\end{tabularx}
\end{table}


\section{$\psi$ is not parameter count: extended evidence}
\label{app:nsc_vs_params}

The main paper (\S\ref{sec:spectral_capacity}) claims that spectral capacity $\psi$ is sensitive to architectural structure that parameter count $mn$ cannot see. We provide four independent lines of evidence: (a)~controlled rank ablation at fixed dimensions, (b)~aspect-ratio dependence at fixed parameter count, (c)~heterogeneous operations at matched dimensions, and (d)~compression scenarios at fixed base architecture. Throughout, $\psi$ is well-defined on rank-deficient matrices since zeros contribute zero to the sum---a property required by the LLaMA pruning of \S\ref{sec:exp:llama} and the rank-truncation analysis below.

\paragraph{(a) Low-rank truncation: same dimensions, same parameter count, different $\psi$.}
The cleanest test of $\psi \neq mn$ is to fix the dimensions $(m, n)$ and vary only the spectrum.
We take a $4096 \times 11008$ \texttt{gate\_proj} matrix from layer 0 of LLaMA-7B and apply rank-$k$ SVD truncation for $k \in \{1, 2, 4, \ldots, 4096\}$.
Each truncated matrix occupies the same $\mathbb{R}^{4096 \times 11008}$ ambient space, so $mn = 45{,}088{,}768$ is constant across all truncations.
Yet $\psi$ varies by more than three orders of magnitude (Figure~\ref{fig:nsc_vs_params}): full rank yields $\psi(W) \approx 3174$, rank 1000 retains only $42.5\%$ of the full capacity ($\psi \approx 1349$), and rank 1 collapses to $\psi \approx 2.5$---a $1{,}270\times$ reduction.
The reduction is monotone in $k$ and tracks reconstruction fidelity $\|W - W_k\|_F / \|W\|_F$ closely, confirming that $\psi$ measures the information content of the spectrum rather than the size of the bounding box.

\begin{figure}[h]
\centering
\includegraphics[width=0.95\linewidth]{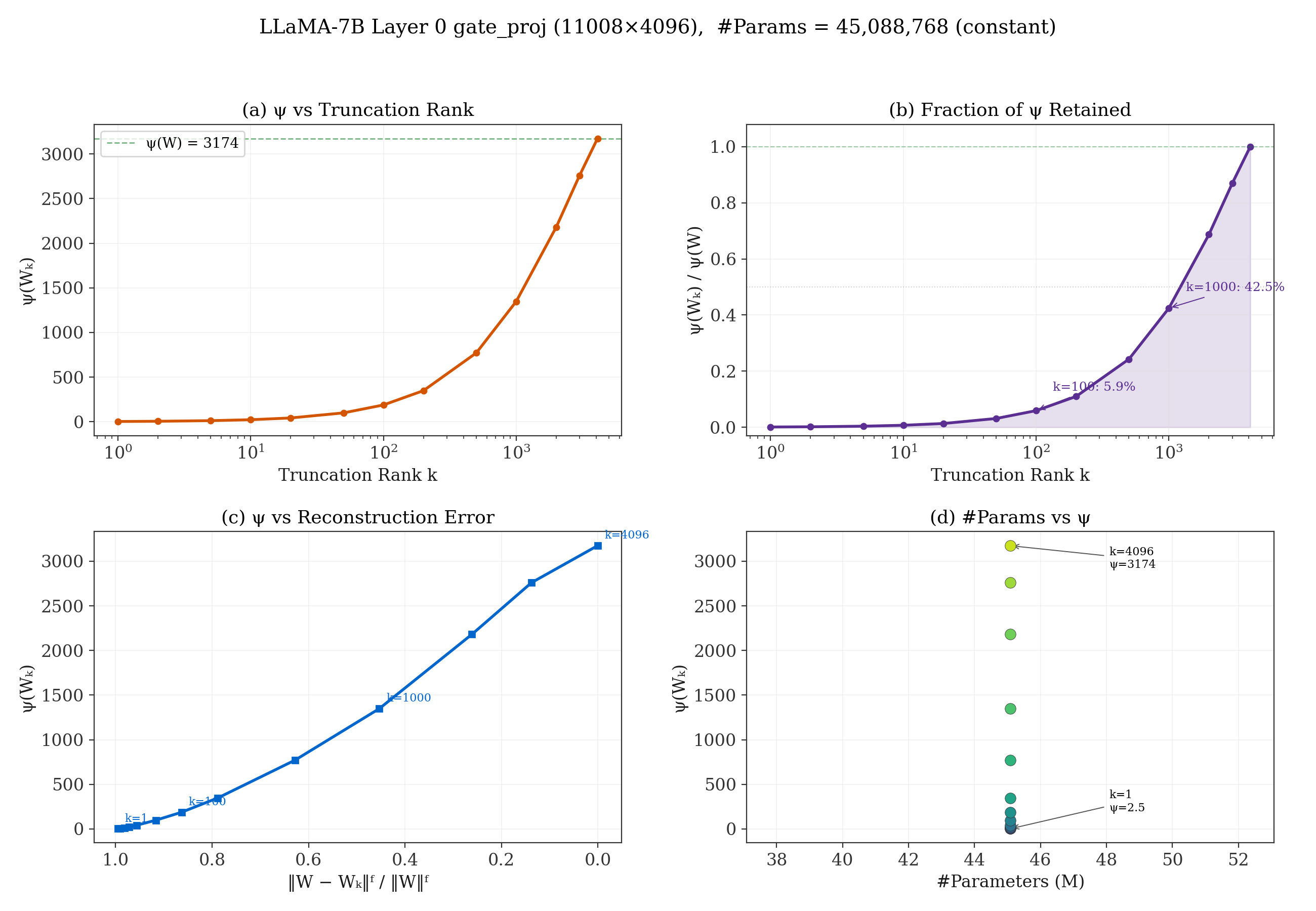}
\caption{$\psi \neq$ parameter count. Low-rank truncation of the $4096 \times 11008$ \texttt{gate\_proj} matrix from LLaMA-7B layer 0. (a)~$\psi$ grows monotonically with truncation rank $k$. (b)~Even at rank 1000, only $42.5\%$ of full capacity is retained. (c)~$\psi$ tracks reconstruction fidelity. (d)~Parameter count is constant ($\approx 45$M) across all ranks while $\psi$ varies by orders of magnitude.}
\label{fig:nsc_vs_params}
\end{figure}

\paragraph{(b) Aspect ratio: same parameter count, different $\psi$.}
Even at fixed parameter count $mn = k$, the value of $\psi$ depends on the aspect ratio $\gamma = \min(m, n) / \max(m, n)$.
Using the closed form $\psi_{\mathrm{MP}}(m, n, s)$ (Eq.~\ref{eq:psi_mp}, \S\ref{sec:closed_form}), we compare matrices analytically: for fixed $k$ and initialization variance $s^2$, the function $m \mapsto \psi_{\mathrm{MP}}(m, k/m, s)$ is maximized at $m = \sqrt{k}$, i.e., among all rectangular matrices with $k$ parameters, the square matrix has the highest spectral capacity.
A geometric intuition for this fact is provided in Appendix~\ref{app:proof_mp}.

A concrete consequence: at $k = 4 \times 10^6$ parameters with Xavier initialization ($s^2 = 2/(m+n)$), a $2000 \times 2000$ square matrix has $\psi_{\mathrm{MP}} \approx 1{,}161$ (versus a Jensen upper bound $N \ln(1 + M s^2) = 2000 \ln 2 \approx 1{,}386$), while a $4000 \times 1000$ rectangular matrix has $\psi_{\mathrm{MP}} \approx 909$ (versus a Jensen upper bound $1000 \ln(1 + 1.6) \approx 956$)---a $\approx 28\%$ gap with identical $mn$.
Both values are computed by 1D Gauss--Kronrod quadrature of the MP integral and verified against the Tulino--Verd\'u closed form (Appendix~\ref{app:psi-concavity}); they sit strictly below their respective Jensen upper bounds, consistent with random-Gaussian rather than perfectly orthogonal spectra.
Parameter count cannot distinguish these architectures; $\psi_{\mathrm{MP}}$ ranks the square one higher, consistent with the empirical observation that balanced layer dimensions tend to yield better-trained models.

\paragraph{(c) Heterogeneous operations: same dimensions, different $\psi$.}
On the FlexiBERT benchmark, layers contain different operation types---self-attention (SA), linear transform (LT), depthwise-separable convolution (DSC)---instantiated with identical hidden size $h$ and FFN width.
Parameter count assigns the same score to layers with the same $(h, \text{FFN})$ regardless of operation type, since $mn$ is determined by dimensions alone.
$\psi$, in contrast, distinguishes operation types because each op type contributes a different set of weight matrices to the layer-wise sum in Eq.~\ref{eq:nsc}: SA contributes Q/K/V projections plus an output projection; LT contributes a single linear map; DSC contributes a depthwise and a pointwise tensor.
The aggregate $\psi$ values differ because the spectral capacity of several smaller matrices is not equal to the spectral capacity of one large matrix at the same total parameter count.

Operation-type structure is the source of FlexiBERT's ranking gap between $\psi$-based proxies and parameter-counting baselines.
Empirically, on the 500-architecture spec-faithful FlexiBERT set (\S\ref{sec:experiments}), \ourmethod{} achieves Spearman $\rho = 0.884$, while \#Params achieves $\rho = 0.676$---a $0.21$-point gap that arises principally from $\psi$'s sensitivity to operation-type structure that \#Params is blind to.

\paragraph{(d) Compression scenarios: same base architecture, different subnet $\psi$.}
A second setting where parameter count fails entirely is structured pruning of a fixed base model.
In the LoNAS experiment (\S\ref{sec:exp:llama}), all subnets share the same LLaMA-7B base architecture: same depth (32 layers), same $d_{\mathrm{model}} = 4096$, same attention configuration.
Subnets differ only in which per-layer FFN dimensions are retained from the base $d_{\mathrm{ffn}} = 11008$.
At a fixed compression budget (e.g., $5.7$B parameters), parameter count is constant by construction across all candidate subnets and provides zero ranking signal.
$\psi$, by contrast, evaluates the spectral content of each pruned weight matrix and ranks subnets by how much spectral capacity the chosen FFN dimensions retain.
This is precisely the regime where $\psi$'s structural sensitivity matters most, and it explains why \ourmethod{}-guided selection achieves the highest average across seven commonsense reasoning benchmarks at the $5.7$B budget (Table~\ref{tab:lonas}).

\paragraph{Summary.}
Across all four regimes---rank ablation, aspect ratio, heterogeneous operations, and compression---$\psi$ discriminates between architectures or weight matrices that parameter count cannot. The first two are controlled mathematical settings in which $mn$ is held exactly constant; the latter two are practical NAS scenarios where the constraint structure makes \#Params uninformative by design. Together, these establish that $\psi$ extracts architectural information genuinely orthogonal to model size.

\section{Choice of aggregation: protocol and analysis}
\label{app:aggregation_choice}


\paragraph{Protocol.}
Table~\ref{tab:agg_ablation} (\S\ref{sec:nsc_definition}) compares the across-layer summation of \ourmethod{} against the two natural alternatives; this appendix details the protocol and explains the results.
We evaluate on the three ranking benchmarks of \S\ref{sec:exp_proxy} under their standard protocols: $500$ FlexiBERT architectures against GLUE (spec-faithful setting, \S\ref{app:flexibert-spec-faithful}), $200$ GPT-2 architectures against WikiText-103 perplexity (negated, so that larger $\rho$ is better), and $1{,}001$ AutoFormer-Tiny architectures against ImageNet Top-1.
Per-matrix capacities and layer capacities $\Psi_l$ are computed once per architecture following \S\ref{sec:nsc_definition}; the rules differ only in how the $\Psi_l$ are combined: summation, the serial bottleneck $\min_l \Psi_l$---the series half of the composition rule for channel capacities~\citep{foggo2023on}---and the product $\prod_l \Psi_l$.
Alongside Spearman $\rho$, Table~\ref{tab:agg_ablation} reports the number of distinct score values each rule assigns: a rule that cannot separate architectures cannot rank them.

\paragraph{Why the bottleneck collapses.}
The failure of $\min_l \Psi_l$ is structural rather than incidental.
Its score always equals one layer's capacity, so a search space admits at most $D$ distinct scores, where $D$ is the number of distinct layer configurations the space allows: $D = 2 \times 3 \times 4 = 24$ for FlexiBERT (hidden size, operation type, FFN configuration), $D = 2 \times 2 \times 2 = 8$ for AutoFormer-Tiny, and $D = 17{,}370$ for GPT-2, whose per-layer FFN width is drawn from a wide integer range.
The measured counts ($22$, $3$, $195$) sit below these bounds because the minimum concentrates on the weakest configurations: an architecture's minimum avoids a weak configuration only if every layer does---unlikely at depth---so on AutoFormer-Tiny $895$ of the $1{,}001$ architectures share a single score and only $3$ of the $8$ configurations ever serve as the minimum.
The same mechanism produces the negative correlation there: adding layers can only lower a minimum, while deeper architectures tend to be more accurate in that space.
GPT-2 is the complementary case: $D$ far exceeds the number of architectures, the bound never binds, and summation still outranks the bottleneck ($0.968$ vs.\ $0.813$).

\paragraph{Why the product conflates depth with capacity.}
With layer capacities well above one, $\prod_l \Psi_l$ grows geometrically in depth, so differences in $L$ overwhelm differences within layers.
The rule is competitive on FlexiBERT ($0.917$), where depth takes only two values, but drops to $0.712$ on GPT-2, where depth varies freely.

\paragraph{Summary.}
The bottleneck fails by construction on spaces with few layer configurations, and the product conflates depth with capacity; summation avoids both failure modes and is the only additive rule---the structural property \ouralgo{} requires (\S\ref{sec:nsc_dp}).

\section{Proofs}

\subsection{Proof of the mutual information identity}
\label{app:proof_mi}

We prove the mutual information identity $I(x; y) = \tfrac{1}{2}\psi(W)$ stated inline in \S\ref{sec:spectral_capacity}.

For $y = Wx + z$ with $x \sim \mathcal{N}(0, I_n)$, $z \sim \mathcal{N}(0, I_m)$ independent, we have $y \sim \mathcal{N}(0, WW^\top + I_m)$ and $y \mid x \sim \mathcal{N}(Wx, I_m)$. Using the differential entropy of $\mathcal{N}(\mu, \Sigma)$ on $\mathbb{R}^d$, $H = \tfrac{1}{2}\ln\det(2\pi e\,\Sigma)$:
\begin{align*}
    I(x; y) &= H(y) - H(y \mid x) \\
    &= \tfrac{1}{2}\ln\det(2\pi e\,(I_m + WW^\top)) - \tfrac{1}{2}\ln\det(2\pi e\, I_m) \\
    &= \tfrac{1}{2}\ln\det(I_m + WW^\top) \;=\; \tfrac{1}{2}\ln\det(I_n + W^\top W) \;=\; \tfrac{1}{2}\psi(W),
\end{align*}
where the third equality uses $\ln\det(2\pi e\, A) = m\ln(2\pi e) + \ln\det(A)$ for $A \in \mathbb{R}^{m\times m}$ (the $m\ln(2\pi e)$ contribution cancels in the difference), the fourth equality uses Sylvester's determinant identity $\det(I_m + AB) = \det(I_n + BA)$, and the final equality uses $\ln\det(I + W^\top W) = \sum_i \ln(1+\sigma_i^2)$, valid because the eigenvalues of $W^\top W$ are the squared singular values of $W$. \qed

\subsection{Gradient Responsiveness}
\label{app:proof_gradient}

\begin{proposition}[Gradient Responsiveness]
\label{prop:gradient}
$\nabla_W \psi(W) = 2W(I + W^\top W)^{-1}$. In the SVD basis, the per-singular-value gradient weight is $g(\sigma) = 2\sigma/(1+\sigma^2)$, which attains its unique maximum at $\sigma = 1$.
\end{proposition}

\begin{proof}
Let $A = W^\top W$. Then $\mathrm{d}\psi = \mathrm{tr}[(I{+}A)^{-1}\,\mathrm{d}A]$ with $\mathrm{d}A = \mathrm{d}W^\top W + W^\top \mathrm{d}W$. Expanding and using cyclicity of trace:
\begin{equation*}
    \mathrm{d}\psi = 2\,\mathrm{tr}[(I{+}A)^{-1}W^\top\, \mathrm{d}W]
\end{equation*}
giving $\nabla_W \psi = 2W(I{+}W^\top W)^{-1}$. In the SVD basis $W = U\Sigma V^\top$:
\begin{equation*}
    \nabla_W \psi = U\,\mathrm{diag}\!\left(\frac{2\sigma_i}{1{+}\sigma_i^2}\right) V^\top
\end{equation*}
so the per-singular-value gradient weight is $g(\sigma) = 2\sigma/(1{+}\sigma^2)$. Setting $g'(\sigma) = 2(1{-}\sigma^2)/(1{+}\sigma^2)^2 = 0$ gives the unique maximum at $\sigma = 1$.
\end{proof}

\paragraph{Remark.} This is the gradient of $\psi$ with respect to its argument $W$, not the gradient of any task loss. During training, $W$ is updated by $\nabla_W \mathcal{L}_{\mathrm{task}}$, which has no direct relationship to $\nabla_W \psi$. What Proposition~\ref{prop:gradient} establishes is that $\psi$ is most sensitive to singular values of order one---the regime where standard variance-preserving initialization (Xavier, Kaiming) places them. The connection between $\psi$-dynamics and task-loss dynamics in Figure~\ref{fig:motivation_psi} is empirical, not derived from this gradient analysis.

\subsection{Proof of Theorem~\ref{thm:mp} (Marchenko-Pastur Spectral Capacity)}
\label{app:proof_mp}

Let $W \in \mathbb{R}^{m \times n}$ have i.i.d.\ entries with mean $0$ and variance $s^2$, along a sequence of sizes with $\min(m, n) \to \infty$ and $\gamma$ and $Ms^2$ held fixed, as in the statement of the theorem. Without loss of generality assume $m \geq n$ (otherwise consider $W^\top$, which has identical singular values), so that $M = m$, $N = n$.

\paragraph{Step 1: Marchenko-Pastur law.}
The normalized Gram matrix $S = \frac{1}{Ms^2}W^\top W \in \mathbb{R}^{N \times N}$ has eigenvalues $\tilde{\lambda}_1 \geq \cdots \geq \tilde{\lambda}_N \geq 0$. By the Marchenko-Pastur theorem \cite{marchenko1967distribution}, the empirical spectral distribution
\begin{equation*}
    \mu_N = \frac{1}{N}\sum_{i=1}^N \delta_{\tilde{\lambda}_i}
\end{equation*}
converges almost surely (in distribution) to the Marchenko-Pastur law $\mu_{\gamma}$ with density:
\begin{equation*}
    f_{\mathrm{MP}}(\lambda; \gamma) = \frac{\sqrt{(\lambda_+ - \lambda)(\lambda - \lambda_-)}}{2\pi\gamma\lambda}, \quad \lambda \in [\lambda_-, \lambda_+]
\end{equation*}
where $\lambda_\pm = (1 \pm \sqrt{\gamma})^2$.

\paragraph{Step 2: From eigenvalues to $\psi$.}
The eigenvalues of $W^\top W$ are $Ms^2\tilde{\lambda}_i$. Since $\psi(W) = \sum_i \ln(1 + \lambda_i(W^\top W))$:
\begin{align*}
    \frac{1}{N}\psi(W) &= \frac{1}{N}\sum_{i=1}^N \ln\!\left(1 + Ms^2\tilde{\lambda}_i\right) = \int \ln\!\left(1 + Ms^2\lambda\right) \mathrm{d}\mu_N(\lambda)
\end{align*}

\paragraph{Step 3: Convergence.}
The integrand $g(\lambda) = \ln(1 + Ms^2\lambda)$ is a fixed function along the sequence, since $Ms^2$ is held fixed. The extreme nonzero eigenvalues of $S$ converge a.s.\ to the support edges $\lambda_\pm$~\citep{bai2010spectral}, so the spectral mass is eventually contained in a fixed compact set on which $g$ is bounded and continuous, and the weak convergence $\mu_N \to \mu_{\gamma}$ (a.s.) gives
\begin{equation*}
    \frac{1}{N}\psi(W) = \int g \, \mathrm{d}\mu_N \;\xrightarrow{\;\mathrm{a.s.}\;}\; \int_{\lambda_-}^{\lambda_+} \ln\!\left(1 + Ms^2\lambda\right) f_{\mathrm{MP}}(\lambda; \gamma)\, d\lambda \;=\; \mathcal{I}(\gamma, Ms^2),
\end{equation*}
which is the claimed limit. \qed

\paragraph{Note on finite-size convergence.}
Under the convention adopted in \S\ref{sec:closed_form}, $\psi_{\mathrm{MP}}$ is the \emph{definition} of $\psi$ for use as an architecture proxy, not an approximation of a finite-sample quantity. Empirical demonstration of finite-size convergence at typical Transformer dimensions is given in Appendix~\ref{app:closed_form_validity}.

\paragraph{Geometric intuition for the aspect-ratio result.}
Appendix~\ref{app:nsc_vs_params}~(b) shows numerically that for fixed $mn$, the square matrix maximizes $\psi_{\mathrm{MP}}$. Intuitively, varying $\gamma$ at fixed $mn$ trades off the number of eigenvalues $N = \sqrt{\gamma\,mn}$ (favoring $\gamma = 1$) against the per-eigenvalue contribution $\ln(1 + Ms^2\lambda)$ where $M = \sqrt{mn/\gamma}$ scales as $1/\sqrt{\gamma}$ (favoring $\gamma \to 0$). The product $\psi_{\mathrm{MP}}(\gamma) = N\,\mathbb{E}_\lambda[\ln(1+Ms^2\lambda)]$ is maximized at $\gamma = 1$, the spectral-theoretic counterpart of the classical result that mutual information across parallel Gaussian channels at fixed total power is maximized by uniform power allocation.

\subsection{Strict Concavity of $\psi_{\mathrm{MP}}$ and Optimality of Uniform FFN Allocation}
\label{app:psi-concavity}

This appendix substantiates the structural claim made in Section~\ref{sec:exp:search}: under a fixed total feed-forward parameter budget, the \ourmethod{}-maximizing per-layer FFN allocation is uniform.
The role of concavity should be stated precisely.
It underwrites the uniform-allocation corollary at the end of this appendix---it explains \emph{what} \ouralgo{} returns on this space.
\ouralgo{}'s global-optimality guarantee does not depend on it: that guarantee follows from layer-wise additivity, exhaustive enumeration of the network-level configurations, and exact solution of the inner knapsack (\S\ref{sec:nsc_dp}), and holds for arbitrary per-layer value tables.
Were $\psi_{\mathrm{MP}}$ not concave, \ouralgo{} would still return the global proxy-maximizer---just not necessarily a uniform one.

\begin{proposition}[Strict concavity of $\psi_{\mathrm{MP}}$ in its dimension arguments]
\label{prop:psi-concavity}
For fixed $m$ and $s > 0$, the map $n \mapsto \psi_{\mathrm{MP}}(n, m, s)$ is strictly concave on $(0, \infty)$---hence strictly concave along the integer grids \ouralgo{} searches---and, by the symmetry of $\psi_{\mathrm{MP}}$ in its two dimension arguments, the same holds in $m$.
No relation between $m$ and $n$ is assumed.
\end{proposition}

In the Transformer-XL instantiation, $n$ plays $d_{\mathrm{ff}}$ and $m$ plays $d_{\mathrm{model}}$; the proposition holds on all of $(0, \infty)$, with no restriction to $d_{\mathrm{ff}} \geq d_{\mathrm{model}}$.

\begin{proof}
Let $\beta \coloneqq s^2$ and write $M = \max(m,n)$, $N = \min(m,n)$, $\gamma = N/M$ as in Theorem~\ref{thm:mp}.

\emph{Step 1: fixed-point representation of $\psi_{\mathrm{MP}}$.}
Consider the pair of equations
\begin{equation}
\label{eq:mp_fixed_point}
\delta \;=\; \frac{1}{1 + \beta n \tilde{\delta}},
\qquad
\tilde{\delta} \;=\; \frac{1}{1 + \beta m \delta}.
\end{equation}
Substituting the second equation into the first shows that a solution must satisfy $G(\delta; n) = 0$, where
\begin{equation*}
G(\delta; n) \;\coloneqq\; \delta\left(1 + \frac{\beta n}{1 + \beta m \delta}\right) - 1 .
\end{equation*}
Since $G(0; n) = -1 < 0$, $\partial G/\partial \delta = 1 + \beta n / (1+\beta m \delta)^2 > 0$, and $G(\delta; n) \to \infty$ as $\delta \to \infty$, Eq.~\eqref{eq:mp_fixed_point} has a unique positive solution $(\delta^*, \tilde{\delta}^*)$ for every $m, n, \beta > 0$.
The pair is the classical Marchenko--Pastur self-consistency in symmetric form: writing $\eta(x) \coloneqq \int f_{\mathrm{MP}}(\lambda; \gamma)/(1 + x\lambda)\, d\lambda$ for the $\eta$-transform of the MP law, the unique positive solution is
\begin{equation}
\label{eq:delta_resolvent}
\delta^* = \frac{m - N + N\,\eta(M\beta)}{m},
\qquad
\tilde{\delta}^* = \frac{n - N + N\,\eta(M\beta)}{n},
\end{equation}
the deterministic equivalents of the normalized resolvent traces $\frac{1}{m}\,\mathrm{tr}\,(I_m + W^{\!\top} W)^{-1}$ and $\frac{1}{n}\,\mathrm{tr}\,(I_n + W W^{\!\top})^{-1}$ for $W \in \mathbb{R}^{n \times m}$ with i.i.d.\ entries of variance $s^2$~\citep[Thm.~2.39]{tulino2004random}.

Define
\begin{equation}
\label{eq:Phi_def}
\Phi(\delta, \tilde{\delta}; n) \;\coloneqq\; n \ln(1 + \beta m \delta) \;+\; m \ln(1 + \beta n \tilde{\delta}) \;-\; \beta m n\, \delta \tilde{\delta}.
\end{equation}
Direct differentiation gives
$\partial \Phi/\partial \delta = \beta m n \left[(1+\beta m \delta)^{-1} - \tilde{\delta}\right]$ and
$\partial \Phi/\partial \tilde{\delta} = \beta m n \left[(1+\beta n \tilde{\delta})^{-1} - \delta\right]$,
so $(\delta, \tilde{\delta})$ is a stationary point of $\Phi(\cdot, \cdot\,; n)$ exactly when it solves Eq.~\eqref{eq:mp_fixed_point}.
We claim the value identity
\begin{equation}
\label{eq:phi_identity}
\psi_{\mathrm{MP}}(n, m, s) \;=\; \Phi(\delta^*, \tilde{\delta}^*; n).
\end{equation}
To derive it, view both sides as functions of $\beta$ at fixed $(m, n)$, with $(\delta^*(\beta), \tilde{\delta}^*(\beta))$ the positive root of Eq.~\eqref{eq:mp_fixed_point}.
At $\beta = 0$ both sides vanish (there $\delta^* = \tilde{\delta}^* = 1$).
For the left side, differentiating the MP integral (Eqs.~\ref{eq:mp_limit} and~\ref{eq:psi_mp}) in $\beta$ gives
\begin{equation*}
\frac{d \psi_{\mathrm{MP}}}{d\beta}
= N \!\int\! \frac{M\lambda}{1 + \beta M \lambda}\, f_{\mathrm{MP}}(\lambda;\gamma)\, d\lambda
= \frac{N\bigl(1 - \eta(M\beta)\bigr)}{\beta}
= \frac{n\bigl(1 - \tilde{\delta}^*\bigr)}{\beta}
= m n\, \delta^* \tilde{\delta}^*,
\end{equation*}
where the third equality uses Eq.~\eqref{eq:delta_resolvent} and the fourth uses $1 - \tilde{\delta}^* = \beta m \delta^* \tilde{\delta}^*$, a rearrangement of the second fixed-point equation.
For the right side, stationarity means only the explicit $\beta$-dependence contributes (envelope theorem):
\begin{equation*}
\frac{d}{d\beta}\,\Phi(\delta^*, \tilde{\delta}^*; n)
= \left.\frac{\partial \Phi}{\partial \beta}\right|_{(\delta^*\!,\, \tilde{\delta}^*)}
= \frac{m n\, \delta^*}{1+\beta m \delta^*} + \frac{m n\, \tilde{\delta}^*}{1+\beta n \tilde{\delta}^*} - m n\, \delta^* \tilde{\delta}^*
= m n\, \delta^* \tilde{\delta}^*,
\end{equation*}
applying the two fixed-point equations to the first two terms.
Equal values at $\beta = 0$ and equal $\beta$-derivatives establish Eq.~\eqref{eq:phi_identity}.
(Numerically, the two sides of Eq.~\eqref{eq:phi_identity} agree to relative error below $10^{-15}$ across the dimension ranges of our benchmarks.)

\emph{Step 2: envelope identity for the marginal capacity.}
Differentiating Eq.~\eqref{eq:phi_identity} in $n$, stationarity again cancels the dependence through $(\delta^*, \tilde{\delta}^*)$:
\begin{equation}
\label{eq:envelope}
\frac{d \psi_{\mathrm{MP}}}{d n}
= \left.\frac{\partial \Phi}{\partial n}\right|_{(\delta^*\!,\, \tilde{\delta}^*)}
= \ln\!\big(1 + \beta m \delta^*\big) + \frac{\beta m \tilde{\delta}^*}{1 + \beta n \tilde{\delta}^*} - \beta m\, \delta^* \tilde{\delta}^*
= \ln\!\big(1 + \beta m\, \delta^*(n)\big),
\end{equation}
the last two terms cancelling by the first fixed-point equation.

\emph{Step 3: monotonicity and strict concavity.}
Implicit differentiation of $G(\delta^*(n); n) = 0$ gives
\begin{equation*}
\frac{d \delta^*}{d n} \;=\; -\,\frac{\partial G/\partial n}{\partial G/\partial \delta}
\;=\; -\,\frac{\beta \delta^* / (1 + \beta m \delta^*)}{\,1 + \beta n/(1+\beta m \delta^*)^2\,} \;<\; 0,
\end{equation*}
hence
\begin{equation*}
\frac{d^2 \psi_{\mathrm{MP}}}{d n^2}
\;=\; \frac{\beta m}{1 + \beta m\, \delta^*(n)}\,\frac{d \delta^*}{d n} \;<\; 0
\qquad \text{for all } n \in (0, \infty).
\end{equation*}
Strict concavity in $m$ follows from the invariance of $\Phi$ under the simultaneous swap $(n, \delta) \leftrightarrow (m, \tilde{\delta})$, mirroring the symmetry $\psi_{\mathrm{MP}}(n, m, s) = \psi_{\mathrm{MP}}(m, n, s)$.
\end{proof}

\paragraph{Remark: the envelope formula is the rank-one mechanism, made deterministic.}
Adding one row $w \sim \mathcal{N}(0, s^2 I_m)$ to $W$ changes $\psi$ by $\ln(1 + w^{\!\top}(I + W^{\!\top} W)^{-1} w)$ (matrix determinant lemma), whose conditional expectation concentrates on $\ln(1 + s^2\, \mathrm{tr}\,(I + W^{\!\top} W)^{-1}) \approx \ln(1 + \beta m \delta^*)$---exactly Eq.~\eqref{eq:envelope}.
The probabilistic argument below runs on this mechanism at finite sizes; the proof above is its deterministic sharpening.

\paragraph{A second, probabilistic proof of (non-strict) discrete concavity.}
Let $W_d \in \mathbb{R}^{d \times m}$ have i.i.d.\ $\mathcal{N}(0, s^2)$ entries, and let $M_d := I + W_d^\top W_d$. Adding one row $w \sim \mathcal{N}(0, s^2 I_{m})$ independent of $W_d$ gives $M_{d+1} = M_d + ww^\top$. By the matrix determinant lemma,
\begin{equation*}
\psi(W_{d+1}) - \psi(W_d) \;=\; \ln\det M_{d+1} - \ln\det M_d \;=\; \ln\bigl(1 + w^\top M_d^{-1} w\bigr).
\end{equation*}
Taking conditional expectation over $w$ (write $w = s\,z$, $z \sim \mathcal{N}(0, I_{m})$),
\begin{equation*}
\mathbb{E}\bigl[\psi(W_{d+1}) - \psi(W_d)\,\bigm|\, W_d\bigr]
\;=\;
\mathbb{E}_{z}\!\left[\ln\!\Bigl(1 + s^2 \sum_{i=1}^{m} \tfrac{z_i^2}{1 + \sigma_i^2(W_d)}\Bigr)\right],
\end{equation*}
where $\sigma_i^2(W_d)$ are the eigenvalues of $W_d^\top W_d$. As $d$ increases, each $\sigma_i^2(W_d)$ grows (rank-1 PSD updates can only increase eigenvalues, by Weyl's inequality), so each coefficient $1/(1 + \sigma_i^2(W_d))$ shrinks. Since $\ln(1 + \cdot)$ is monotone increasing, the expected marginal gain $\mathbb{E}[\psi(W_{d+1}) - \psi(W_d)]$ is non-increasing in $d$: the map $d \mapsto \mathbb{E}[\psi(W_d)]$ is discretely concave at every finite size.

Passing from this statement to the concavity of $\psi_{\mathrm{MP}}$ requires care: by the central limit theorem for linear spectral statistics~\citep{bai2010spectral}, $\psi(W) - \psi_{\mathrm{MP}}$ is $O_P(1)$, not $o(1)$---only the relative deviation vanishes (Appendix~\ref{app:mp_convergence})---so a limit taken along $n$ cannot be exchanged with a second difference in $n$.
The passage is completed by taking the limit along the family that realizes $\psi_{\mathrm{MP}}$ at \emph{finite} $(m, n)$.
Because $\psi_{\mathrm{MP}}$ depends on the specification only through $(N, \gamma, Ms^2)$, it is exactly $1$-homogeneous under the replica map $(n, m, s) \mapsto (kn, km, s/\sqrt{k})$, which fixes both $\gamma$ and $Ms^2$:
$\psi_{\mathrm{MP}}(kn, km, s/\sqrt{k}) = k\,\psi_{\mathrm{MP}}(n, m, s)$ for every $k \in \mathbb{N}$.
Along this family the hypotheses of Theorem~\ref{thm:mp} hold with $\gamma$ and $Ms^2$ \emph{fixed}, so
\begin{equation*}
\tfrac{1}{k}\,\psi\bigl(W^{(k)}\bigr) \;=\; N \cdot \frac{\psi(W^{(k)})}{kN} \;\xrightarrow{\;\mathrm{a.s.}\;}\; N\,\mathcal{I}(\gamma, Ms^2) \;=\; \psi_{\mathrm{MP}}(n, m, s)
\qquad (k \to \infty),
\end{equation*}
with $W^{(k)}$ of size $kn \times km$ and i.i.d.\ entries of variance $s^2/k$.
Uniform integrability upgrades this to convergence of expectations: $0 \le \psi(W) = \ln\det(I + W^\top W) \le \mathrm{tr}(W^\top W) = \|W\|_F^2$, and $\tfrac{1}{k}\|W^{(k)}\|_F^2$ is an average of i.i.d.\ variables with expectation $mns^2$ for every $k$, converging a.s.\ to that constant, so Pratt's generalized dominated convergence theorem gives $\mathbb{E}\bigl[\tfrac{1}{k}\psi(W^{(k)})\bigr] \to \psi_{\mathrm{MP}}(n, m, s)$.
The rank-one argument applies unchanged at each $k$: incrementing $n$ by one adds $k$ rows at fixed column count and entry variance, and a sum of $k$ consecutive non-increasing marginal gains is non-increasing, so every $n \mapsto \mathbb{E}\bigl[\tfrac{1}{k}\psi(W^{(k)})\bigr]$ is discretely concave.
Discrete concavity at a grid point is a single closed inequality among three values, so it passes to the pointwise limit: $n \mapsto \psi_{\mathrm{MP}}(n, m, s)$ is discretely concave. $\square$

\paragraph{Numerical check of the implementation.}
As an independent check that the quadrature implementation of $\psi_{\mathrm{MP}}$ exhibits the concavity proven in Proposition~\ref{prop:psi-concavity}, Figure~\ref{fig:psi-concavity-numerical} plots the marginal gain $\partial\psi_{\mathrm{MP}}/\partial d_{\mathrm{ff}}$ across all $(d_{\mathrm{model}}, d_{\mathrm{ff}})$ pairs in our Transformer-XL search space: all seven curves are strictly decreasing, and all $189$ discrete second differences are strictly negative. Table~\ref{tab:psi-marginal-gain} reports the marginal gain at the two endpoints of the search range: it decreases by roughly $2\times$ for every value of $d_{\mathrm{model}}$, consistent with the proposition.

\begin{figure}[H]
\centering
\includegraphics[width=0.7\linewidth]{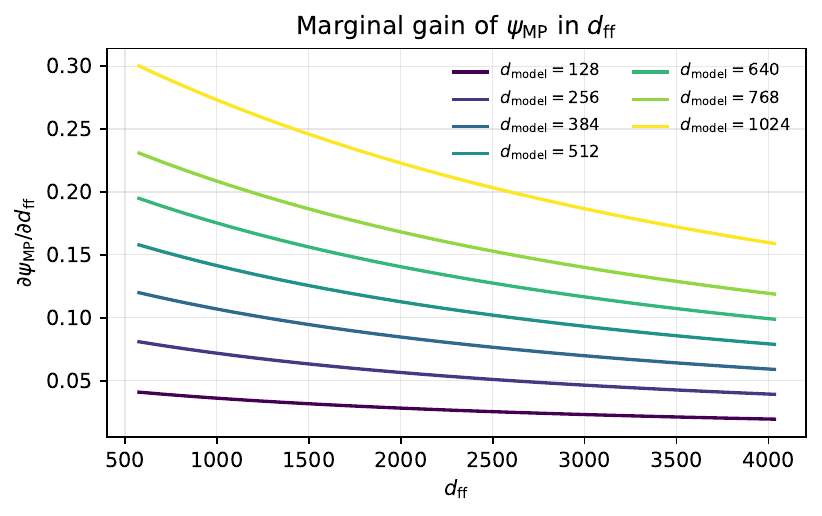}
\caption{\textbf{Implementation check of Proposition~\ref{prop:psi-concavity}.} Marginal gain $\partial\psi_{\mathrm{MP}}/\partial d_{\mathrm{ff}}$ (forward differences over a $d_{\mathrm{ff}}$ grid of step $128$) as a function of $d_{\mathrm{ff}}$ for each of the seven embedding dimensions in the Transformer-XL search space, computed at $s=0.02$. All seven curves are strictly decreasing, with all $189$ second differences of $\psi_{\mathrm{MP}}(d_{\mathrm{ff}}, d_{\mathrm{model}}, s)$ negative across the full range, as the proposition requires.}
\label{fig:psi-concavity-numerical}
\end{figure}

\begin{table}[h]
\centering
\caption{Marginal gain $\partial\psi_{\mathrm{MP}}/\partial d_{\mathrm{ff}}$ at the two endpoints of the FFN search range, for each embedding dimension. The marginal gain attenuates by roughly $2\times$ across the range, consistent with Proposition~\ref{prop:psi-concavity}.}
\label{tab:psi-marginal-gain}
\setlength{\tabcolsep}{6pt}
\begin{tabular}{c c c c}
\toprule
$d_{\mathrm{model}}$ & at $d_{\mathrm{ff}} = 512$ & at $d_{\mathrm{ff}} = 4096$ & all 2nd diff.\ ${<}\,0$? \\
\midrule
128  & 0.041 & 0.020 & \checkmark \\
256  & 0.081 & 0.040 & \checkmark \\
384  & 0.120 & 0.060 & \checkmark \\
512  & 0.158 & 0.080 & \checkmark \\
640  & 0.195 & 0.100 & \checkmark \\
768  & 0.231 & 0.120 & \checkmark \\
1024 & 0.300 & 0.161 & \checkmark \\
\bottomrule
\end{tabular}
\end{table}

\paragraph{Corollary (Uniform allocation optimality).}
Under a fixed total FFN parameter budget $\sum_{l=1}^{L} d_{\mathrm{ff}, l} \cdot (2 d_{\mathrm{model}} + 1) \leq B$, the unique \ourmethod{}-maximizing per-layer allocation $\{d_{\mathrm{ff}, l}\}_{l=1}^{L}$ of the continuous relaxation is uniform, $d_{\mathrm{ff}, l} = B / [L (2 d_{\mathrm{model}} + 1)]$ for all $l$; on the discrete search grid the optimum is its rounding. This is Jensen's inequality applied to Proposition~\ref{prop:psi-concavity}: the objective $\sum_l \psi_{\mathrm{MP}}(d_{\mathrm{ff}, l}, d_{\mathrm{model}}, s)$ is a sum of strictly concave functions of the $d_{\mathrm{ff}, l}$ under a linear constraint on $\sum_l d_{\mathrm{ff}, l}$, so replacing any non-uniform allocation by its average strictly increases it. The uniform allocation discovered by \ouralgo{} in Section~\ref{sec:exp:search} is therefore not coincidental but a structural consequence of \ourmethod{}'s spectral form. $\square$


\section{Multi-Budget Search via \ouralgo{}}
\label{app:multiobj}

\subsection{Motivation}

We elaborate on the multi-budget property stated in \S\ref{sec:nsc_dp}. Because (i) the resource axis (parameters or, equivalently in
the LoNAS architecture, TFLOPs) is an integer sum of per-block contributions
and (ii) the spectral-capacity score $\psi_{\mathrm{MP}}$ is per-layer additive,
the DP table after the final layer encodes
\begin{equation}
\mathcal{P}_{\ouralgo{}}
\;=\;
\Bigl\{\bigl(b,\;\textstyle\max_{\text{subnet}\,:\,\textsc{Params}\le b}\;\textstyle\sum_\ell \Psi_\ell(r_\ell, h_\ell)\bigr)
\;:\; b \in \mathcal{B}\Bigr\},
\end{equation}
where $\Psi_\ell(r_\ell, h_\ell)$ is the layer capacity (\S\ref{sec:nsc_definition}) at the chosen $(r_\ell, h_\ell)$ configuration. $\mathcal{P}_{\ouralgo{}}$ is the entire Pareto front over the discretized budget grid $\mathcal{B}$.
Switching the deployment budget therefore requires \emph{neither additional
search nor a new proxy cache}; the practitioner reads a different cell of
the same DP table.

\subsection{Empirical demonstration: seven subnets from a single DP run}

To make this property concrete, we instantiate the seven operating points
(``Subnet-A'' through ``Subnet-G'') reported in
\citet{munoz2024lonas}'s Table~4 by reading them off a \emph{single}
\ouralgo{} execution. No re-search, no per-budget retraining, and no
per-budget proxy re-computation is performed: all seven configurations are
obtained from the same $\psi_{\mathrm{MP}}$ cache and the same DP backward
pass used in \S\ref{sec:exp:llama}. A--G correspond to LoNAS's published TFLOPs operating points, while the main-text result is read at the $5.7$\,B parameter budget---a different point on the same Pareto front.

Table~\ref{tab:appendix_AG_pareto} reports the resulting subnets, their
TFLOPs (computed at the conventional $T\!=\!256$ sequence length used
throughout the LoNAS literature), per-task accuracy under our unified
evaluation protocol, and the amortised search wall-clock time per operating
point. The LoNAS-SuperNet (maximal configuration, $r_\ell\!=\!32$,
$h_\ell\!=\!11008$ for all $\ell$) is included as the upper-bound reference.

\begin{table}[!htbp]
\centering
\caption{
\textbf{Seven operating points along the \ouralgo{} Pareto front,
extracted from a single DP backward pass.}
TFLOPs are computed at sequence length $T\!=\!256$.
All seven \ouralgo{} rows share \emph{the same} $0.46$\,s DP backward pass:
they are not seven independent searches but seven read-outs of one
DP table. The marginal cost of obtaining an additional operating
point is therefore \emph{exactly zero}; switching budgets is a
$\mathcal{O}(B)$ table look-up. The LoNAS-SuperNet row is the
maximal-configuration upper bound and performs no search.}
\label{tab:appendix_AG_pareto}
\setlength{\tabcolsep}{4pt}
\renewcommand{\arraystretch}{1.05}
\footnotesize
\resizebox{\textwidth}{!}{%
\begin{tabular}{l c c c c c c c c c c c}
\toprule
\textbf{Subnet} & \textbf{TFLOPs} &
BoolQ & PIQA & SIQA & HSwag & WGrnd & ARC-e & ARC-c & OBQA &
\textbf{Avg$_8$} & \textbf{Search} \\
\midrule
LoNAS-SuperNet  & 1.71 & 66.02 & 77.58 & 72.93 & 57.49 & 66.85 & 78.41 & 62.29 & 76.20 & \textbf{69.72} & --- \\
\midrule
\ouralgo{}-A & 1.41 & 64.01 & 72.42 & 68.17 & 51.26 & 64.09 & 72.52 & 58.53 & 70.60 & 65.20 & \multirow{7}{*}{\shortstack[c]{\textbf{0.46\,s}\\\emph{(shared,}\\\emph{single DP)}}} \\
\ouralgo{}-B & 1.40 & 63.73 & 72.42 & 68.27 & 51.05 & 63.46 & 72.18 & 58.79 & 71.20 & 65.14 & \\
\ouralgo{}-C & 1.39 & 63.61 & 72.25 & 67.50 & 50.69 & 63.38 & 72.18 & 58.87 & 72.00 & 65.06 & \\
\ouralgo{}-D & 1.38 & 63.73 & 72.03 & 66.73 & 49.95 & 63.85 & 71.34 & 58.19 & 70.80 & 64.58 & \\
\ouralgo{}-E & 1.37 & 63.15 & 71.49 & 67.71 & 49.73 & 63.69 & 71.30 & 57.08 & 71.00 & 64.39 & \\
\ouralgo{}-F & 1.36 & 62.51 & 69.37 & 66.63 & 46.54 & 62.51 & 68.06 & 54.86 & 67.80 & 62.29 & \\
\ouralgo{}-G & 1.29 & 59.27 & 69.48 & 64.64 & 44.73 & 62.27 & 66.04 & 52.82 & 66.60 & 60.73 & \\
\bottomrule
\end{tabular}%
}
\end{table}

Figure~\ref{fig:pareto_main} (main paper) visualizes the corresponding accuracy--TFLOPs Pareto front; the monotone shape is a structural consequence of the per-layer additivity of $\psi_{\mathrm{MP}}$ (adjacent budgets produce adjacent subnets, removing the run-to-run variance of stochastic searches).

\subsection{Discussion}

Two observations follow from Table~\ref{tab:appendix_AG_pareto}.

\paragraph{(i) The DP table is the Pareto front.}
The seven \ouralgo{} rows are not seven independent searches: they are
seven read-outs of the same $0.46$\,s DP backward pass. Every
\emph{additional} budget the practitioner queries costs $\mathcal{O}(B)$
table look-ups (microseconds), not a new search.

\paragraph{(ii) Smooth, monotone front.}
Avg$_8$ decreases monotonically from $65.20\%$ at $1.41$\,TFLOPs
(\ouralgo{}-A) to $60.73\%$ at $1.29$\,TFLOPs (\ouralgo{}-G)---a
$4.47$\,pp range across $0.12$\,TFLOPs of compute reduction.
The smoothness of this front is itself a downstream consequence of the
closed-form, additive structure of $\psi_{\mathrm{MP}}$: small changes in
the budget produce small, predictable changes in the recovered subnet,
removing the run-to-run variance that plagues stochastic
searches.

\paragraph{Summary.}
\ouralgo{} returns Pareto-front output as a \emph{free} structural consequence of its closed-form, additive formulation rather than as the output of an explicit per-budget search. Among the methods evaluated in this paper, \ouralgo{} is the only one that returns an entire Pareto front in a single sub-second invocation while requiring neither labels nor GPU forward passes.


\section{Extended Experimental Results}
\label{app:realdata}

\subsection{FlexiBERT: original (Serianni) vs.\ spec-faithful setting}
\label{app:flexibert-spec-faithful}

The 500-architecture FlexiBERT benchmark used in \S\ref{sec:exp_proxy} originates from \citet{tuli2022flexibert}, who defined a heterogeneous Transformer search space in which each architecture independently varies four dimensions: hidden size $h \in \{128, 256\}$, depth $L \in \{2, 4\}$, \emph{per-layer} feed-forward dimension $d_{\mathrm{ff}, l} \in \{512, 1024\}$, and per-layer attention operation type. \citet{serianni2023nasbench} subsequently retrained these architectures on GLUE to obtain ground-truth scores, and this is the standard ranking benchmark adopted by all subsequent training-free Transformer proxies (W-PCA~\citep{wang2025wpca}, ZeroLM~\citep{chen2025zerolm}, attention-based proxies in~\citet{serianni2023nasbench}).

\paragraph{The issue with the original (Serianni) evaluation protocol.}
The proxy-evaluation code released by \citet{serianni2023nasbench}, and reused by subsequent training-free proxies that report numbers on this benchmark, instantiates \emph{every} candidate architecture with two of the four search dimensions hard-coded to fixed constants. Concretely, regardless of the architecture's declared $h$ and $\{d_{\mathrm{ff}, l}\}$, the model is built with $\texttt{hidden\_size}=256$ and $\texttt{intermediate\_size}=1024$ for all layers (see the original \texttt{build\_flexibert} implementation; the same constants are reproduced whenever subsequent proxy papers reuse \texttt{build\_flexibert} as the model factory).
This collapses two of the four search dimensions: in our census of the 500-architecture set, $251/500$ architectures declare $h=128$ in their specification but are scored at $h=256$, and the per-layer FFN structure---one of the two dimensions FlexiBERT was specifically designed to expose---is replaced by a constant.
The downstream effect is that any proxy whose value depends on the architecture specification (parameter count, FLOPs, W-PCA, \ourmethod{}, etc.) is computed on a \emph{scrambled} spec rather than the architecture whose GLUE score is being predicted. Proxies that score from instantiated weights or activations are similarly affected, since the model they instantiate no longer matches the architecture that was trained.

\paragraph{The spec-faithful setting.}
We restore the search-space-faithful protocol: each of the 500 architectures is scored with its declared $(h, \{d_{\mathrm{ff}, l}\})$, recovering the full $\{128, 256\} \times \{512, 1024\}^L$ structure across the 500 specs. The set of architectures, the GLUE ground truth, and the proxy implementations are otherwise identical to the original setting. We refer to this as the \emph{Spec-faithful} setting and use it as the default in \S\ref{sec:exp_proxy}; results under the original \emph{Serianni} setting are reported below for direct comparability with prior tables.

\paragraph{Empirical effect: structure-aware proxies improve, structure-blind ones do not.}
Table~\ref{tab:flexibert-serianni-vs-spec-faithful} reports Kendall $\tau$ and Spearman $\rho$ against ground-truth GLUE for representative proxies under both settings. Two observations: \emph{(i)}~Architectural quantities improve under the spec-faithful protocol, with \ourmethod{} gaining substantially ($\tau$: $0.544 \to 0.695$) and \#Params modestly ($0.454 \to 0.485$). \emph{(ii)}~Among training-free proxies, responses diverge: W-PCA, SNIP, GradNorm, and SynFlow recover from suppressed values (W-PCA: $0.522 \to 0.638$; SynFlow: $-0.107 \to 0.028$), consistent with their reported numbers being evaluated on a population whose two of four declared dimensions are collapsed; ZeroLM, conversely, \emph{decreases} ($0.543 \to 0.527$), suggesting its prior lead drew partly on the collapsed evaluation rather than on architectural signal. Together these patterns confirm that the original protocol systematically depresses the measurable signal of structure-sensitive proxies and that the spec-faithful setting, not Serianni, is the appropriate apples-to-apples ranking benchmark for the FlexiBERT search space as originally defined.

\begin{table}[h]
\centering
\small
\caption{\textbf{FlexiBERT ranking quality: Serianni vs.\ spec-faithful protocol.} Same 500 architectures and GLUE ground truth; only the architecture spec passed to the proxy differs. Architectural quantities (\ourmethod{}, \#Params) improve under the spec-faithful protocol; among training-free proxies, W-PCA, SNIP, GradNorm, and SynFlow recover from suppressed values while ZeroLM \emph{decreases}. \textbf{Bold}: best per column; \underline{underline}: second best.}
\label{tab:flexibert-serianni-vs-spec-faithful}
\setlength{\tabcolsep}{6pt}
\begin{tabular}{l cc cc}
\toprule
& \multicolumn{2}{c}{\textbf{Serianni (original)}} & \multicolumn{2}{c}{\textbf{Spec-faithful (ours)}} \\
\cmidrule(lr){2-3}\cmidrule(lr){4-5}
\textbf{Method} & $\tau$ & $\rho$ & $\tau$ & $\rho$ \\
\midrule
\rowcolor{rowhl}
\textbf{\ourmethod{} (ours)}              & \textbf{0.544} & \textbf{0.776} & \textbf{0.695} & \textbf{0.884} \\
\#Params                                  & 0.454 & 0.652 & 0.485 & 0.676 \\
\midrule
W-PCA~\citep{wang2025wpca}                & 0.522 & 0.757 & \underline{0.638} & \underline{0.850} \\
ZeroLM~\citep{chen2025zerolm} & \underline{0.543} & \underline{0.775} & 0.527 & 0.730 \\
SNIP~\citep{lee2019snip}                  & 0.096 & 0.146 & 0.275 & 0.411 \\
GradNorm~\citep{abdelfattah2021zerocost}  & 0.013 & 0.025 & 0.210 & 0.316 \\
SynFlow~\citep{tanaka2020pruning}         & $-0.107$ & $-0.159$ & 0.028 & 0.040 \\
\bottomrule
\end{tabular}
\end{table}

\subsection{Controlled-correlation analysis across architecture families}
\label{app:controlled-correlation}

To test whether \ourmethod{} contributes ranking signal beyond \#Params and \#FLOPs, we report two complementary controls on every benchmark in our evaluation: a non-parametric \emph{windowed} Kendall $\tau$ and \emph{Kendall partial correlation}~\citep{kendall1942partial,conover1981rank}.

\paragraph{Windowed control.}
For each control variable $Z \in \{\#\mathrm{Params}, \#\mathrm{FLOPs}\}$, we restrict $\tau$ to architecture pairs $(i,j)$ with $|Z_i - Z_j|/\max(Z_i, Z_j) < 10\%$, removing $Z$ as a discriminating signal.

\paragraph{Kendall partial correlation (Conover--Iman, used uniformly).}
Windowed control is non-parametric but uses only a sample-sparse subset (${\sim}10\%$ of pairs) and depends on the threshold. As a complementary, threshold-free analysis we compute Kendall partial correlation $\tilde\tau$ via the Conover--Iman~\citep{conover1981rank} rank-OLS estimator: rank-transform all signals, OLS-residualize the proxy and ground-truth ranks against the controlled ranks, and apply Kendall $\tau$ to the residuals. We use this estimator uniformly for both single-variable controls ($Z = \#\mathrm{Params}$ or $\#\mathrm{FLOPs}$) and joint $(\#\mathrm{Params}, \#\mathrm{FLOPs})$ control, ensuring that columns are mathematically comparable and that no estimator-switch artifacts can creep in across columns. (For reference, Kendall's (1942) classical closed-form $\tilde\tau_{XY\cdot Z} = (\tau_{XY} - \tau_{XZ}\tau_{YZ})/\sqrt{(1-\tau_{XZ}^2)(1-\tau_{YZ}^2)}$ produces values that differ from the Conover--Iman estimate by at most $\pm 0.06$ on every entry of \autoref{tab:flexibert-controlled}; we report Conover--Iman as the primary number for cross-column consistency.)

\paragraph{Per-benchmark caveats.}
\emph{GPT-2:} the LiteTransformerSearch GPT-2 search space varies $d_{\mathrm{model}}$ and $L$ at fixed FFN ratio, so $\#\mathrm{Params}$ and $\#\mathrm{FLOPs}$ are rank-equivalent across all architectures (their pairwise $\tau$ is exactly $1$); single-variable Kendall partial controls collapse the two columns to identical values. \emph{CNN benchmarks (NATS-Bench-SSS, MobileNetV3):} ZeroLM is Transformer-specific and therefore not applicable.

\paragraph{FlexiBERT.}
Under \#Params-windowed control, \#Params itself collapses to $\tau=0.082$ (near random) while \ourmethod{} retains $\tau=0.505$. Under \#FLOPs-windowed control, \#FLOPs collapses to $\tau\approx -0.03$ while \ourmethod{} retains $\tau=0.298$. Under joint $(\#\mathrm{Params}, \#\mathrm{FLOPs})$ control, \#Params and \#FLOPs themselves collapse to $\tilde\tau \le 0.06$ as expected, while \ourmethod{} retains $\tilde\tau = 0.477$, exceeding the strongest training-free competitor W-PCA ($0.444$) by $0.033$, with ZeroLM further behind at $0.217$. This confirms that \ourmethod{} captures architectural information beyond \#Params and \#FLOPs---the structural signal that size and compute miss.

\begin{table}[ht]
\centering
\footnotesize
\caption{%
  \textbf{Controlled-correlation analysis on FlexiBERT.}
  Methods grouped as in the main-text \autoref{tab:flexibert-correlation}; FlexiBERT is scored under the spec-faithful protocol of \S\ref{app:flexibert-spec-faithful}.
  \emph{Windowed $\tau$ (${<}10\%$):} Kendall $\tau$ on architecture pairs whose \#Params (resp.\ \#FLOPs) differ by less than $10\%$.
  \emph{Kendall partial $\tilde\tau$:} \textbf{Conover--Iman~\citep{conover1981rank} rank-OLS estimator used uniformly} for single-variable and joint controls (see text). ``---'' marks self-control (NaN by definition; numerical OLS residuals are reported as $\le 0.05$ noise and treated as undefined).
  \textbf{Bold}: best; \underline{underline}: second best.%
}
\label{tab:flexibert-controlled}
\setlength{\tabcolsep}{3.5pt}
\renewcommand{\arraystretch}{0.98}
\begin{tabular}{l c c c c c c}
\toprule
& & \multicolumn{2}{c}{\textbf{Windowed} $\tau$ (${<}10\%$)} & \multicolumn{3}{c}{\textbf{Kendall partial} $\tilde\tau$ \textbf{(Conover--Iman)}} \\
\cmidrule(lr){3-4}\cmidrule(lr){5-7}
\textbf{Method} & $\tau$ & $\mid\#$P & $\mid\#$F & $\mid\#$P & $\mid\#$F & $\mid\#$P,$\#$F \\
\midrule
\rowcolor{rowhl}
\textbf{\ourmethod{} (ours)} & \textbf{0.695}    & \textbf{0.505}    & \underline{0.298}    & \textbf{0.584}    & \textbf{0.477}    & \textbf{0.477}    \\
\#Params                     & 0.485             & 0.082             & 0.273                & ---               & 0.237             & 0.013             \\
\#FLOPs                      & 0.552             & 0.335             & $-0.030$             & 0.412             & ---               & 0.053             \\
\midrule
W-PCA                        & \underline{0.635} & \underline{0.417} & \textbf{0.341}    & \underline{0.508} & \underline{0.474} & \underline{0.444} \\
ZeroLM                       & 0.527             & 0.356             & 0.052                & 0.412             & 0.220             & 0.217             \\
SNIP                         & 0.289             & 0.237             & $-0.061$             & 0.258             & 0.073             & 0.109             \\
GradNorm                     & 0.171             & 0.164             & $-0.021$             & 0.173             & 0.077             & 0.100             \\
\bottomrule
\end{tabular}
\end{table}

\paragraph{GPT-2.}
Under \#Params-windowed control \ourmethod{} retains $\tau=0.536$, narrowly leading W-PCA ($0.531$) and tied with ZeroLM ($0.536$); under \#FLOPs-windowed control \ourmethod{} leads at $\tau=0.436$. Under joint $(\#\mathrm{Params}, \#\mathrm{FLOPs})$ control all signals collapse the same way (single-variable and joint columns are degenerate; see caveat above) and \ourmethod{} retains $\tilde\tau=0.329$, behind W-PCA ($0.382$) and ZeroLM ($0.377$).

\begin{table}[H]
\centering
\footnotesize
\caption{\textbf{Controlled-correlation analysis on GPT-2 (LiteTransformerSearch).} Same protocol as \autoref{tab:flexibert-controlled}.}
\label{tab:gpt2-controlled}
\setlength{\tabcolsep}{3.5pt}
\renewcommand{\arraystretch}{0.98}
\begin{tabular}{l c c c c c c}
\toprule
& & \multicolumn{2}{c}{\textbf{Windowed} $\tau$ (${<}10\%$)} & \multicolumn{3}{c}{\textbf{Kendall partial} $\tilde\tau$ \textbf{(Conover--Iman)}} \\
\cmidrule(lr){3-4}\cmidrule(lr){5-7}
\textbf{Method} & $\tau$ & $\mid\#$P & $\mid\#$F & $\mid\#$P & $\mid\#$F & $\mid\#$P,$\#$F \\
\midrule
\rowcolor{rowhl}
\textbf{\ourmethod{} (ours)} & 0.849             & \textbf{0.536}    & \textbf{0.436}    & 0.329             & 0.329             & 0.329             \\
\#Params                     & 0.854             & 0.530             & 0.395             & ---               & ---               & 0.214             \\
\#FLOPs                      & 0.854             & 0.530             & 0.395             & ---               & ---               & 0.214             \\
\midrule
W-PCA                        & \textbf{0.878}    & 0.531             & 0.428             & \textbf{0.382}    & \textbf{0.382}    & \textbf{0.382}    \\
ZeroLM                       & 0.841             & \underline{0.536} & \underline{0.431} & \underline{0.377} & \underline{0.377} & \underline{0.377} \\
SNIP                         & 0.856             & 0.451             & 0.312             & 0.130             & 0.130             & 0.130             \\
GradNorm                     & \underline{0.869} & 0.448             & 0.311             & 0.312             & 0.312             & 0.312             \\
\bottomrule
\end{tabular}
\end{table}

\paragraph{AutoFormer-Tiny.}
\ourmethod{} is the top-ranked method in every column of the table: under \#Params-windowed control \ourmethod{} retains $\tau=0.503$ vs.\ $0.322$ for \#Params; under joint $(\#\mathrm{Params},\#\mathrm{FLOPs})$ control it retains $\tilde\tau=0.439$, well above the second-best ZeroLM ($0.242$). Gradient-based proxies (SNIP, GradNorm) are at or near zero correlation under every controlled column.

\begin{table}[H]
\centering
\footnotesize
\caption{\textbf{Controlled-correlation analysis on AutoFormer-Tiny.} Same protocol as \autoref{tab:flexibert-controlled}.}
\label{tab:autoformer-t-controlled}
\setlength{\tabcolsep}{3.5pt}
\renewcommand{\arraystretch}{0.98}
\begin{tabular}{l c c c c c c}
\toprule
& & \multicolumn{2}{c}{\textbf{Windowed} $\tau$ (${<}10\%$)} & \multicolumn{3}{c}{\textbf{Kendall partial} $\tilde\tau$ \textbf{(Conover--Iman)}} \\
\cmidrule(lr){3-4}\cmidrule(lr){5-7}
\textbf{Method} & $\tau$ & $\mid\#$P & $\mid\#$F & $\mid\#$P & $\mid\#$F & $\mid\#$P,$\#$F \\
\midrule
\rowcolor{rowhl}
\textbf{\ourmethod{} (ours)} & \textbf{0.610}    & \textbf{0.503}    & \textbf{0.494}    & \textbf{0.324}    & \textbf{0.317}    & \textbf{0.439}    \\
\#Params                     & 0.466             & 0.322             & 0.296             & ---               & 0.033             & $-0.041$          \\
\#FLOPs                      & \underline{0.512} & 0.385             & 0.359             & 0.187             & ---               & $-0.147$          \\
\midrule
W-PCA                        & 0.500             & 0.365             & 0.342             & 0.205             & 0.132             & 0.114             \\
ZeroLM                       & 0.506             & \underline{0.428} & \underline{0.424} & \underline{0.251} & \underline{0.133} & \underline{0.242} \\
SNIP                         & 0.095             & $-0.014$          & $-0.028$          & $-0.096$          & $-0.109$          & $-0.070$          \\
GradNorm                     & $-0.059$          & $-0.082$          & $-0.090$          & $-0.111$          & $-0.107$          & $-0.080$          \\
\bottomrule
\end{tabular}
\end{table}

\paragraph{AutoFormer-Small.}
The pattern of AutoFormer-Tiny replicates: \ourmethod{} ranks first in every column, with $\tau=0.803$ overall, $\tau=0.423$ under \#Params-windowed control, and $\tilde\tau=0.253$ under joint partial control (vs.\ $0.159$ for ZeroLM and ${\le}0.086$ for \#Params, \#FLOPs, and W-PCA).

\begin{table}[H]
\centering
\footnotesize
\caption{\textbf{Controlled-correlation analysis on AutoFormer-Small.} Same protocol as \autoref{tab:flexibert-controlled}.}
\label{tab:autoformer-s-controlled}
\setlength{\tabcolsep}{3.5pt}
\renewcommand{\arraystretch}{0.98}
\begin{tabular}{l c c c c c c}
\toprule
& & \multicolumn{2}{c}{\textbf{Windowed} $\tau$ (${<}10\%$)} & \multicolumn{3}{c}{\textbf{Kendall partial} $\tilde\tau$ \textbf{(Conover--Iman)}} \\
\cmidrule(lr){3-4}\cmidrule(lr){5-7}
\textbf{Method} & $\tau$ & $\mid\#$P & $\mid\#$F & $\mid\#$P & $\mid\#$F & $\mid\#$P,$\#$F \\
\midrule
\rowcolor{rowhl}
\textbf{\ourmethod{} (ours)} & \textbf{0.803}    & \textbf{0.423}    & \textbf{0.419}    & \textbf{0.205}    & \textbf{0.249}    & \textbf{0.253}    \\
\#Params                     & 0.789             & 0.366             & 0.364             & ---               & 0.053             & 0.086             \\
\#FLOPs                      & 0.787             & 0.361             & 0.359             & $-0.019$          & ---               & $-0.039$          \\
\midrule
W-PCA                        & 0.785             & 0.353             & 0.352             & 0.064             & 0.083             & 0.062             \\
ZeroLM                       & \underline{0.797} & \underline{0.406} & \underline{0.397} & \underline{0.171} & \underline{0.190} & \underline{0.159} \\
SNIP                         & 0.697             & 0.121             & 0.132             & $-0.022$          & $-0.018$          & $-0.022$          \\
GradNorm                     & 0.541             & 0.009             & 0.022             & $-0.019$          & $-0.018$          & $-0.019$          \\
\bottomrule
\end{tabular}
\end{table}

\paragraph{NATS-Bench-SSS (CIFAR-100).}
\ourmethod{} retains the highest correlation in every column of the table despite being applied without modification to convolutional channel widths: $\tau=0.372$ under \#Params-windowed control (vs.\ $0.116$ for \#Params), and $\tilde\tau=0.470$ under joint partial control (vs.\ $0.305$ for the next-best W-PCA, and ${\le}0.030$ for \#Params and \#FLOPs).

\begin{table}[H]
\centering
\footnotesize
\caption{\textbf{Controlled-correlation analysis on NATS-Bench-SSS (CIFAR-100).} Same protocol as \autoref{tab:flexibert-controlled}.}
\label{tab:nats-sss-controlled}
\setlength{\tabcolsep}{3.5pt}
\renewcommand{\arraystretch}{0.98}
\begin{tabular}{l c c c c c c}
\toprule
& & \multicolumn{2}{c}{\textbf{Windowed} $\tau$ (${<}10\%$)} & \multicolumn{3}{c}{\textbf{Kendall partial} $\tilde\tau$ \textbf{(Conover--Iman)}} \\
\cmidrule(lr){3-4}\cmidrule(lr){5-7}
\textbf{Method} & $\tau$ & $\mid\#$P & $\mid\#$F & $\mid\#$P & $\mid\#$F & $\mid\#$P,$\#$F \\
\midrule
\rowcolor{rowhl}
\textbf{\ourmethod{} (ours)} & \textbf{0.706}    & \textbf{0.372}    & \textbf{0.767}    & \textbf{0.377}    & \textbf{0.765}    & \textbf{0.470}    \\
\#Params                     & 0.653             & 0.116             & 0.698             & ---               & 0.683             & 0.028             \\
\#FLOPs                      & 0.250             & $-0.340$          & 0.024             & $-0.335$          & ---               & 0.030             \\
\midrule
W-PCA                        & \underline{0.674} & \underline{0.254} & \underline{0.721} & \underline{0.255} & \underline{0.715} & \underline{0.305} \\
SNIP                         & 0.496             & 0.002             & 0.486             & $-0.006$          & 0.493             & 0.206             \\
GradNorm                     & 0.289             & 0.005             & 0.178             & 0.007             & 0.186             & 0.141             \\
\bottomrule
\end{tabular}
\end{table}

\paragraph{MobileNetV3 (ImageNet, $r{=}224$).}
The Once-for-All (OFA)~\citep{cai2020once} supernet exposes $2.0 \times 10^{19}$ subnets via per-stage depth, per-layer channel expansion ratio, and per-layer kernel size; we use the standard $2040$-architecture evaluation subset. \ourmethod{} treats each depthwise-separable convolution as two reshaped 2D matrices (depthwise $\widetilde{W} \in \mathbb{R}^{c \times k^2}$ and pointwise $\widetilde{W} \in \mathbb{R}^{c_{\mathrm{out}} \times c_{\mathrm{in}}}$) and applies $\psi_{\mathrm{MP}}$ to each. \ourmethod{} again ranks first in five of six columns: $\tau=0.812$ overall and $\tau=0.535$ under \#Params-windowed control (vs.\ $0.414$ for \#Params and $0.436$ for the next-best SNIP). Under joint partial control \ourmethod{} retains $\tilde\tau=0.278$, the highest among methods evaluated, with SNIP second ($0.211$).

\begin{table}[H]
\centering
\footnotesize
\caption{\textbf{Controlled-correlation analysis on MobileNetV3 r=224 (Once-for-All).} Same protocol as \autoref{tab:flexibert-controlled}.}
\label{tab:mnv3-controlled}
\setlength{\tabcolsep}{3.5pt}
\renewcommand{\arraystretch}{0.98}
\begin{tabular}{l c c c c c c}
\toprule
& & \multicolumn{2}{c}{\textbf{Windowed} $\tau$ (${<}10\%$)} & \multicolumn{3}{c}{\textbf{Kendall partial} $\tilde\tau$ \textbf{(Conover--Iman)}} \\
\cmidrule(lr){3-4}\cmidrule(lr){5-7}
\textbf{Method} & $\tau$ & $\mid\#$P & $\mid\#$F & $\mid\#$P & $\mid\#$F & $\mid\#$P,$\#$F \\
\midrule
\rowcolor{rowhl}
\textbf{\ourmethod{} (ours)} & \textbf{0.812}    & \textbf{0.535}    & \underline{0.665} & \textbf{0.439}    & \textbf{0.668}    & \textbf{0.278}    \\
\#Params                     & \underline{0.780} & 0.414             & \textbf{0.668}    & ---               & \underline{0.680} & 0.025             \\
\#FLOPs                      & 0.610             & 0.379             & 0.187             & 0.392             & ---               & 0.011             \\
\midrule
W-PCA                        & 0.630             & 0.292             & 0.400             & 0.234             & 0.408             & 0.137             \\
SNIP                         & 0.687             & \underline{0.436} & 0.392             & \underline{0.434} & 0.397             & \underline{0.211} \\
GradNorm                     & 0.577             & 0.296             & 0.231             & 0.304             & 0.242             & 0.099             \\
\bottomrule
\end{tabular}
\end{table}

\subsection{AutoFormer-Small and AutoFormer-Base end-to-end search}
\label{app:autoformer-search}

The AutoFormer search space is released in three size tiers (Tiny, Small, Base); the main paper (\S\ref{sec:exp:search}, Table~\ref{tab:nas-results}) reports end-to-end search on AutoFormer-Tiny. Here we extend the same protocol to AutoFormer-Small and AutoFormer-Base.

\paragraph{Setup.} For each tier, \ouralgo{} is instantiated on the multiple-choice knapsack (network-level $g$ is depth and embedding dimension; layer-level $x_l$ is per-layer MLP ratio and head count) and selects the proxy-optimal architecture under the corresponding parameter budget (matching the AutoFormer baseline's parameter count for each tier). Ground-truth top-1 / top-5 accuracies are evaluated using the official AutoFormer supernet for each tier. Baselines match \S\ref{sec:exp:search}: the AutoFormer oracle~\citep{chen2021autoformer}, TF-TAS~\citep{zhou2022tftas}, AZ-NAS~\citep{lee2024aznas}, and W-PCA~\citep{wang2025wpca}.

\paragraph{Results.} \autoref{tab:autoformer-small-base} reports accuracy and search cost on AutoFormer-Small and AutoFormer-Base. \ouralgo{}'s search completes in $0.345$\,s (Small) and $0.600$\,s (Base) on a single CPU core---between $6$ and $7$ orders of magnitude faster than the AutoFormer oracle ($24$ GPU-days) and $3$ to $5$ orders of magnitude faster than every training-free proxy baseline. On Acc@1, \ouralgo{} places second on Small ($81.392$, $0.258$ points behind the oracle) and third on Base ($82.086$, $0.008$ points behind TF-TAS); on Acc@5, \ouralgo{} ranks second on Base ($95.688$, $0.058$ points behind the oracle) and third on Small ($95.644$, $0.07$ points behind the oracle).

\paragraph{Ranking quality.} Per-method ranking on AutoFormer-Small is reported in \autoref{tab:autoformer-s-controlled} (\S\ref{app:controlled-correlation}); \ourmethod{} ranks first across all controlled columns.

\begin{table}[ht]
\centering
\caption{Architecture search on the AutoFormer-Small and AutoFormer-Base supernets. AutoFormer-Tiny results are reported in main-paper Table~\ref{tab:nas-results}. Search cost in natural units ($\mathrm{d} =$ GPU days, $\mathrm{s} =$ GPU seconds); AutoFormer / TF-TAS / AZ-NAS costs are reported by the respective authors, W-PCA and \ouralgo{} are wall-clock. \ouralgo{} runs on a single CPU core; baselines run on GPU. \textbf{Bold}: best per column within each Setting block; \underline{underline}: second best.}
\label{tab:autoformer-small-base}
\setlength{\tabcolsep}{6pt}
\renewcommand{\arraystretch}{1.05}
\footnotesize
\begin{tabular}{l l c c c c}
\toprule
\textbf{Setting} & \textbf{Method} & \textbf{\#Params (M)} & \textbf{Acc@1} & \textbf{Acc@5} & \textbf{Search} \\
\midrule
\multirow{5}{*}{Small}
 & AutoFormer~\citep{chen2021autoformer}     & 22.9 & \textbf{81.650}    & \textbf{95.712}       & $24$\,d        \\
 & TF-TAS~\citep{zhou2022tftas}              & 22.8 & 81.344             & \underline{95.698}    & $0.5$\,d       \\
 & AZ-NAS~\citep{lee2024aznas}               & 23.0 & 81.182             & 95.636                & $5{,}184$\,s   \\
 & W-PCA~\citep{wang2025wpca}                & 22.8 & 81.296             & 95.628                & $560.61$\,s    \\
 \rowcolor{rowhl}
 & \textbf{\ouralgo{} (ours)}                & 23.0 & \underline{81.392} & 95.644                & $\mathbf{0.345}$\,\textbf{s} \\
\midrule
\multirow{5}{*}{Base}
 & AutoFormer~\citep{chen2021autoformer}     & 54.0 & \textbf{82.384}    & \textbf{95.746}       & $24$\,d        \\
 & TF-TAS~\citep{zhou2022tftas}              & 54.0 & \underline{82.094} & 95.652                & $0.5$\,d       \\
 & AZ-NAS~\citep{lee2024aznas}               & 53.7 & 81.914             & 95.588                & $9{,}504$\,s   \\
 & W-PCA~\citep{wang2025wpca}                & 53.9 & 81.928             & 95.594                & $1{,}001.84$\,s \\
 \rowcolor{rowhl}
 & \textbf{\ouralgo{} (ours)}                & 53.9 & 82.086             & \underline{95.688}    & $\mathbf{0.600}$\,\textbf{s} \\
\bottomrule
\end{tabular}
\end{table}

\subsection{NATS-Bench-SSS end-to-end search}
\label{sec:exp_nats_sss}

NATS-Bench distinguishes \emph{topology} spaces (varying connectivity and operator types) from \emph{size} spaces (varying widths and depths under a fixed operator skeleton)~\citep{dong2021nats}. By construction, \ourmethod{} captures the spectral capacity of any weight matrix, so it applies to size-type spaces regardless of whether the underlying operator is attention, FFN, or convolution. NATS-Bench-SSS---a pure CNN size space---is the natural test of this claim.

\paragraph{Setup.}
NATS-Bench-SSS contains $32{,}768$ CNN architectures parametrized by the channel count of each of five convolutional layers, with ground-truth accuracies on CIFAR-10, CIFAR-100, and ImageNet16-120. Baselines comprise a weight-sharing supernet (TAS) and training-free proxies (ZiCo, SNIP, SynFlow, GradNorm, NASWOT, RBFleX-NAS, W-PCA, EProxy, TE-NAS); citations appear in Table~\ref{tab:nats-sss-search}. Following the protocol of RBFleX-NAS~\citep{rbflexnas2025}, each training-free baseline randomly samples $1{,}000$ architectures, scores them with the proxy, and reports the ground-truth accuracy of the architecture with the highest proxy score (mean $\pm$ std over multiple sampling seeds); we cite these numbers directly, except for W-PCA~\citep{wang2025wpca}, which we re-run under the same protocol over three seeds. \ouralgo{} instead scores \emph{all} $32{,}768$ candidates exhaustively in $0.54$\,s on a single CPU core. For the comparison in Table~\ref{tab:nats-sss-search} we report the best ground-truth accuracy among the top-$10$ \ouralgo{}-ranked candidates (paying $10$ oracle queries vs.\ the $1$ paid by baselines on their proxy's top architecture); the strict top-$1$ value (paying a single oracle query, matching the baseline protocol) is only marginally lower on each column.

\begin{table}[!h]
\centering
\caption{%
  \textbf{Architecture search on NATS-Bench-SSS}~\citep{dong2021nats}. All times in seconds.
  Oracle (top row) reports the best achievable accuracy on each dataset.
  Baseline numbers from~\citep{rbflexnas2025} ($1{,}000$-sample proxy protocol, mean $\pm$ std over multiple seeds); W-PCA re-run by us under the same protocol over three seeds.
  \ouralgo{} entries are the best ground-truth among its top-$10$ proxy-ranked candidates from the full $32{,}768$-architecture exhaustive scan; the strict top-$1$ value is only marginally lower on each column (see \S\ref{sec:exp_nats_sss} for protocol discussion).
  \ouralgo{} is deterministic with proxy cost independent of dataset (the same $0.54$\,s appears in all three Time columns); \ouralgo{} runs on a single CPU core, baselines run on GPU.
  \textbf{Bold}: best per column (excluding Oracle); \underline{underline}: second best.%
}
\label{tab:nats-sss-search}
\setlength{\tabcolsep}{2.5pt}
\small
\begin{tabular}{l cc cc cc}
\toprule
&
\multicolumn{2}{c}{\textbf{CIFAR-10}} &
\multicolumn{2}{c}{\textbf{CIFAR-100}} &
\multicolumn{2}{c}{\textbf{ImageNet16-120}} \\
\cmidrule(lr){2-3}\cmidrule(lr){4-5}\cmidrule(lr){6-7}
\textbf{Method} &
Acc.\,\% & Time\,(s) &
Acc.\,\% & Time\,(s) &
Acc.\,\% & Time\,(s) \\
\midrule
\rowcolor{gray!8}
Oracle                                   & 93.65 & --- & 71.34 & --- & 47.40 & --- \\
\midrule
TAS~\citep{dong2019tas}                  & \underline{93.40}{\scriptsize$\pm$0.00} & 1946 & \textbf{70.72}{\scriptsize$\pm$0.00} & 3863 & \textbf{47.17}{\scriptsize$\pm$0.00} & 7776 \\
ZiCo~\citep{li2023zico}                  & 90.03{\scriptsize$\pm$0.26} & 69.9 & 69.36{\scriptsize$\pm$0.60} & 69.8 & 44.90{\scriptsize$\pm$0.79} & 1327 \\
SNIP~\citep{lee2019snip}                 & 90.31{\scriptsize$\pm$0.11} & 621 & 69.27{\scriptsize$\pm$0.79} & 640 & 45.25{\scriptsize$\pm$0.32} & 828 \\
SynFlow~\citep{tanaka2020pruning}        & 86.96{\scriptsize$\pm$1.35} & 37.3 & 51.24{\scriptsize$\pm$3.83} & 37.2 & 28.85{\scriptsize$\pm$3.43} & 127 \\
GradNorm~\citep{abdelfattah2021zerocost} & 88.40{\scriptsize$\pm$1.04} & 34.4 & 66.85{\scriptsize$\pm$1.88} & 33.1 & 40.17{\scriptsize$\pm$2.96} & 410 \\
NASWOT~\citep{mellor2021naswot}          & 90.10{\scriptsize$\pm$0.35} & \underline{29.5} & 69.05{\scriptsize$\pm$1.44} & \underline{30.6} & 44.86{\scriptsize$\pm$0.84} & \underline{76.3} \\
RBFleX-NAS~\citep{rbflexnas2025}         & 93.16{\scriptsize$\pm$0.24} & 40.5 & 70.36{\scriptsize$\pm$0.11} & 40.2 & 46.07{\scriptsize$\pm$0.61} & 84.0 \\
W-PCA~\citep{wang2025wpca}               & 93.21{\scriptsize$\pm$0.12} & 282 & 69.95{\scriptsize$\pm$0.32} & 273 & 43.29{\scriptsize$\pm$3.49} & 1351 \\
EProxy~\citep{li2023eproxy}              & 85.76{\scriptsize$\pm$1.05} & 989 & 60.74{\scriptsize$\pm$5.65} & 991 & 34.80{\scriptsize$\pm$1.65} & 8548 \\
TE-NAS~\citep{chen2021tenas}             & 85.42{\scriptsize$\pm$2.76} & 821 & 54.79{\scriptsize$\pm$4.36} & 795 & 37.39{\scriptsize$\pm$1.63} & 15946 \\
\midrule
\rowcolor{rowhl}
\textbf{\ouralgo{} (ours)} & \textbf{93.44} & \textbf{0.54} & \underline{70.68} & \textbf{0.54} & \underline{46.87} & \textbf{0.54} \\
\bottomrule
\end{tabular}
\end{table}

\paragraph{Results.}
\ourmethod{} scores all $32{,}768$ candidates in $0.54$\,s ($\sim 16$\,$\mu$s each)---more than one order of magnitude faster than every baseline (up to four), enabling \emph{exhaustive} ranking instead of the $1{,}000$-sample protocol every other training-free method must use. Despite being derived for Transformer attention/FFN matrices, \ourmethod{} is best on CIFAR-10 ($93.44\%$, ahead of TAS by $0.04$) and second on CIFAR-100 ($70.68\%$) and ImageNet16-120 ($46.87\%$), trailing TAS by ${\le}0.30$ points on those two while running $4{,}000$--$14{,}000\times$ faster than TAS, and reaching within $0.21$/$0.66$/$0.53$ points of the oracle on CIFAR-10/100/ImageNet16-120. By contrast, the strongest Transformer-side ZCP W-PCA collapses to $43.29\pm 3.49\%$ on ImageNet16-120 ($3.58$ points below \ourmethod{}), and CNN-native baselines also fail to match: NASWOT reaches only $44.86\%$ on ImageNet16-120 (vs.\ $46.87\%$ for \ourmethod{}), and RBFleX-NAS trails \ourmethod{} by $0.28$/$0.32$/$0.80$ points across the three datasets. Because \ourmethod{} is defined on $W \in \mathbb{R}^{m \times n}$ independently of how $W$ is used, scaling decisions transfer cleanly from attention/FFN matrices to convolutional channel widths---heuristics whose construction ties them to a specific operator family transfer with materially worse accuracy.

\subsection{Structural properties of \ouralgo{}-discovered architectures}
\label{app:discovered-arch}

This subsection summarizes the structural patterns of the architectures returned by \ouralgo{} in \S\ref{sec:exp:search} and \S\ref{sec:exp:llama}, and links them to the proxy-level analytical results in Appendix~\ref{app:psi-concavity}.

\paragraph{Transformer-XL: uniform per-layer FFN allocation.}
Under the parameter budget $38.4$\,M matching TXL Base, \ouralgo{} returns an $18$-layer $d_{\mathrm{model}} = 512$ architecture with FFN width identical at every layer ($d_{\mathrm{ff}, l} \equiv 1{,}152$ for $l = 1, \ldots, 18$). This uniform allocation is not a coincidence of the search: it is provably the proxy-optimum on this space. Concretely, Appendix~\ref{app:psi-concavity} proves that $\psi_{\mathrm{MP}}(d_{\mathrm{ff}}, d_{\mathrm{model}}, s)$ is strictly concave in $d_{\mathrm{ff}}$ at fixed $d_{\mathrm{model}}$ and $s$, so under the additive parameter constraint $\sum_l d_{\mathrm{ff}, l}(2 d_{\mathrm{model}} + 1) \le B$, Jensen's inequality forces the per-layer FFN-width allocation that maximizes $\sum_l 2\,\psi_{\mathrm{MP}}(d_{\mathrm{ff}, l}, d_{\mathrm{model}}, s)$ to be uniform (rounded to the nearest grid point). The discovered architecture is therefore the Jensen-optimal subnet, and the human-designed TXL Base---which already adopts a uniform $d_{\mathrm{ff}}$ across layers---differs from it only in the choice of $(d_{\mathrm{model}}, L)$, which \ouralgo{} re-selects via its outer enumeration of the network-level configuration $g$.

\paragraph{AutoFormer-Tiny: depth-leaning architectures with conservative head allocation.}
On AutoFormer-Tiny under the $5.7$\,M budget (\autoref{tab:nas-results}), \ouralgo{}'s outer enumeration of network-level $(d_{\mathrm{model}}, L)$ tends to prefer the deeper, narrower configurations (larger $L$ at fixed $d_{\mathrm{model}}$) over the shallow-wide alternatives at the same parameter count. The mechanistic reason follows an analogous (informal) concavity argument: when the per-layer FFN/attention multiple-choice options exhibit decreasing marginal returns in layer-level resource consumption, accumulating more layers contributes more to the additive proxy than concentrating capacity in a few wide layers. The per-layer head-count choices selected by the inner multiple-choice knapsack are typically conservative (the smaller of the available options at each layer), reflecting that \ourmethod{} discounts redundant near-orthogonal projections.

\paragraph{LoNAS-LLaMA-7B: nearly-uniform FFN truncation, full LoRA rank.}
On the LoNAS-LLaMA-7B pruning task at the $5.7$\,B budget (\autoref{tab:lonas}), \ouralgo{}'s solutions select $r_\ell = 32$ at every block (the larger of the two LoRA rank options) and choose per-block FFN intermediate dimensions $h_\ell$ that are nearly uniform across the $32$ blocks. Both choices follow from the proxy structure: keeping $r_\ell$ at its maximum maximizes the LoRA-projection contribution to $\psi_{\mathrm{MP}}$ (which is monotone in the projection dimensions), and the FFN-width concavity of \autoref{app:psi-concavity} again favors a uniform allocation across blocks. Appendix~\ref{app:multiobj} reports these block-by-block decisions for all seven Pareto-optimal operating points (Subnet-A through Subnet-G) extracted from the same $0.46$\,s DP backward pass.

\paragraph{Summary.}
Across all three search settings, the architectures discovered by \ouralgo{} are explainable from the analytical structure of $\psi_{\mathrm{MP}}$: layer-wise concavity yields uniform allocations, and the outer enumeration over network-level decisions selects the depth-width trade-off that maximizes the additive proxy. The proxy-level Jensen-optimal architecture aligns structurally with the human-designed TXL Base on FFN allocation, while \ouralgo{}'s outer enumeration of $(d_{\mathrm{model}}, L)$ further improves perplexity; extending the same logic to LoNAS-LLaMA-7B yields a systematically better pruned subnet at the matched $5.7$\,B budget.


\section{Validating the closed-form \texorpdfstring{$\psi_{\mathrm{MP}}$}{psi-MP}}
\label{app:closed_form_validity}
\subsection{Finite-size Convergence of $\psi_{\mathrm{MP}}$}
\label{app:mp_convergence}

Theorem~\ref{thm:mp} states that the normalized spectral capacity
$\psi(W)/N$ converges almost surely to $\mathcal{I}(\gamma, Ms^2)$ as
$\min(m,n) \to \infty$; its finite-size evaluation
$\psi_{\mathrm{MP}}(m, n, s) = N\,\mathcal{I}(\gamma, Ms^2)$ is the
definition of \ourmethod{} in practice (\S\ref{sec:closed_form}).
This appendix quantifies how closely $\psi_{\mathrm{MP}}$ tracks the
empirical $\psi(W) = \sum_i \ln(1 + \sigma_i^2)$ at the matrix sizes
encountered in contemporary Transformer search spaces.

\paragraph{Protocol.}
For each combination of $\min(m,n) \in \{32, 64, 128, 256, 512, 1024,
2048, 4096\}$, aspect ratio $\gamma = n/m \in \{1.0,\, 0.5,\, 0.25\}$,
and initialization convention
(Xavier $s{=}\sqrt{2/(m{+}n)}$,
 Kaiming fan-in $s{=}\sqrt{2/n}$,
 truncated normal $s{=}0.02$),
we sample $K$ Gaussian matrices
$W \sim \mathcal{N}(0, s^2)^{m \times n}$
and compute
\[
\psi_W = \tfrac{1}{K}\sum_{k=1}^{K}\sum_{i=1}^{\min(m,n)}
\ln\!\big(1 + \sigma_i(W^{(k)})^2\big),
\qquad
\psi_{\mathrm{MP}} = \psi_{\mathrm{MP}}(m, n, s),
\]
using $K{=}20$ samples for $\min(m,n) \le 512$, $K{=}10$ for $1024$,
$K{=}5$ for $2048$, and $K{=}3$ for $4096$.
We report the relative error
$|\psi_{\mathrm{MP}} - \psi_W| / |\psi_W|$.

\begin{figure}[h]
\centering
\includegraphics[width=\linewidth]{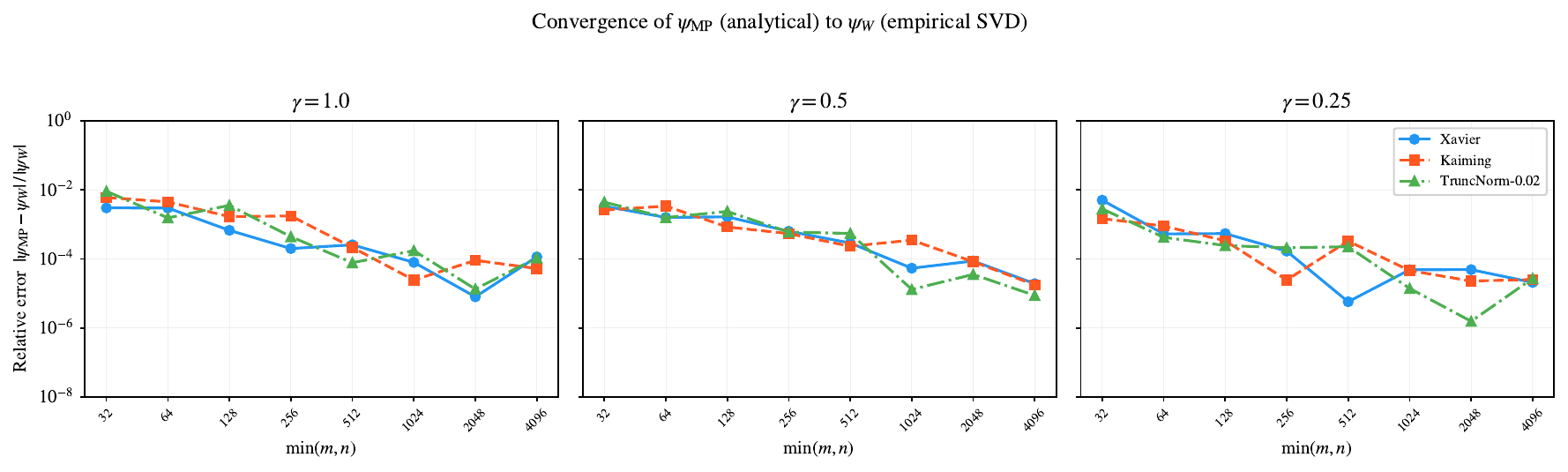}
\caption{%
  Relative error $|\psi_{\mathrm{MP}} - \psi_W| / |\psi_W|$ as a function
  of matrix size, for three aspect ratios $\gamma = n/m$ and three
  initialization conventions.
  The error decays roughly as $1/\min(m,n)$ across all settings.
  By $\min(m,n) = 256$ it is already below $0.2\%$, and at $\min(m,n) =
  4096$ it sits at the level of empirical-SVD sampling noise
  ($\le 10^{-4}$).
}
\label{fig:mp_convergence}
\end{figure}

\begin{table}[h]
\centering
\caption{%
  Relative error of $\psi_{\mathrm{MP}}$ vs.\ $\psi_W$ at representative
  matrix sizes, averaged over all three initialization conventions.
  The error decays smoothly with $\min(m,n)$ and is independent of the
  aspect ratio.
}
\label{tab:mp_convergence}
\small
\begin{tabular}{rcccc}
\toprule
$\min(m,n)$ & $\gamma{=}1.0$ & $\gamma{=}0.5$ & $\gamma{=}0.25$ & overall \\
\midrule
   32 & 6.1e-03 & 3.5e-03 & 3.1e-03 & 4.2e-03 \\
   64 & 3.0e-03 & 2.2e-03 & 6.2e-04 & 1.9e-03 \\
  128 & 2.0e-03 & 1.6e-03 & 3.7e-04 & 1.3e-03 \\
  256 & 8.0e-04 & 5.9e-04 & 1.4e-04 & 5.1e-04 \\
  512 & 1.8e-04 & 3.6e-04 & 1.9e-04 & 2.4e-04 \\
 1024 & 9.3e-05 & 1.4e-04 & 3.6e-05 & 9.0e-05 \\
 2048 & 3.8e-05 & 6.9e-05 & 2.5e-05 & 4.4e-05 \\
 4096 & 9.2e-05 & 1.5e-05 & 2.5e-05 & 4.4e-05 \\
\bottomrule
\end{tabular}
\end{table}

\paragraph{Observations.}
\begin{itemize}
\item \textbf{Sub-1\% error already at $\min(m,n) = 128$.}
  All nine $(\gamma, s)$ combinations achieve relative error
  below $4{\times}10^{-3}$ once the smaller dimension reaches 128;
  the smallest hidden width in any of our six benchmark search spaces
  is FlexiBERT's $H = 128$, so the closed form is already accurate at
  the lower end of practical Transformer scales.
\item \textbf{$O(1/\min(m,n))$ decay.}
  On a log-log plot (Fig.~\ref{fig:mp_convergence}), all curves are
  approximately linear with slope $-1$, consistent with the
  $O(1/N)$ convergence rate predicted by finite-size corrections to
  the Marchenko--Pastur law~\citep{bai2010spectral}.
\item \textbf{Aspect-ratio insensitivity.}
  Squarer matrices ($\gamma=1$) converge slightly slower than tall ones
  ($\gamma=0.25$), but the difference disappears beyond $\min(m,n) =
  512$.
\item \textbf{Initialization-convention insensitivity.}
  All three $s$ choices produce overlapping error curves
  (Fig.~\ref{fig:mp_convergence}); convergence is governed by matrix
  shape rather than variance.
\end{itemize}

\paragraph{Implication.}
Across our six benchmarks the smallest matrix dimension is
$H = 128$ (FlexiBERT) and the largest is $H = 4096$ (LoNAS-LLaMA);
in this entire range the relative error of the closed-form $\psi_{\mathrm{MP}}$
is below 1\%, and at least an order of magnitude smaller than the differences in \ourmethod{} between candidate architectures across our search spaces.
The closed-form integral therefore provides a faithful drop-in
replacement for the empirical SVD-based score, while removing the
need for matrix instantiation, weight sampling, or per-architecture
randomness.


\subsection{Closed-form $\psi_{\mathrm{MP}}$ vs.\ Trained-Weight $\psi_W$: A Consistency Ablation}
\label{app:mp_vs_svd}

\paragraph{The architectural assumption and a fair concern.}
The Marchenko--Pastur (MP) closed form
$\psi_{\mathrm{MP}}(m,n,s)$ used throughout this paper
is derived under an i.i.d.-Gaussian-entry assumption on the underlying
$m\!\times\!n$ matrix~\citep{bai2010spectral}. Trained transformer weights
are not i.i.d.: they exhibit heavy-tailed singular spectra, low-rank
structure induced by SGD/Adam dynamics, and additional structure imposed
by LoRA factorisation. A natural concern is therefore that
$\psi_{\mathrm{MP}}$, which depends only on the architectural shape
$(m,n,s)$, may rank trained-weight subnets differently from a proxy that
actually inspects the trained singular spectrum.

We resolve this concern empirically by running \ouralgo{} with two
interchangeable proxy back-ends on the same search space and comparing
every output.

\paragraph{Setup.}

\emph{Search space.} The full LoNAS-LLaMA-7B elastic space described
in the main text: per-block LoRA rank $r_\ell \in \{32, 28\}$ and FFN width
$h_\ell \in \{11008,9632,8256,6880,5504\}$ over $L\!=\!32$ blocks, with
parameter budget ranging from $4.57$\,B to $6.74$\,B. Both back-ends produce
the same set of $129$ Pareto-optimal subnets (one per FFN-sum bin).

\emph{Back-end (a) -- $\psi_{\mathrm{MP}}$ (closed form).}
For every distinct $(m,n)$ shape that appears in the search space we evaluate
\begin{equation}
\psi_{\mathrm{MP}}(m,n,s) \;=\; N \!\int_{\lambda_-}^{\lambda_+}\!\ln\!\bigl(1 + Ms^{2}\lambda\bigr)\,f_{\mathrm{MP}}(\lambda;\gamma)\,d\lambda,
\end{equation}
with $M=\max(m,n)$, $N=\min(m,n)$, $\gamma=N/M$, and Xavier variance $s^2 = 2/(m+n)$, by 1-D Gauss--Kronrod quadrature.
For the LoNAS-LLaMA-7B search space only $13$ distinct $(m,n)$ tuples
arise, so the entire cache is built in $0.5$\,s on a single CPU core
\emph{without loading any pretrained checkpoint}.

\emph{Back-end (b) -- $\psi_W$ (trained-weight SVD).}
For every $(\text{layer}, \text{projection}, \text{choice})$ triple in the
search space we load the trained LoNAS-LLaMA-7B supernet, materialise the
corresponding effective matrix (LoRA factors merged into the base
weight), compute its singular values via
\texttt{torch.linalg.svdvals}, and record
$\psi_W = \sum_i \ln(1 + \sigma_i^2)$. The cache requires the full
pretrained checkpoint and takes $\sim 6$ minutes on a single NVIDIA RTX 5090.

\emph{Search.} Identical \ouralgo{} (exact dynamic programming) is run on
each cache; both runs return the per-bin top-1 over the entire achievable
parameter range.

\paragraph{Metrics.}
We pre-register four agreement metrics and one fidelity metric:

\begin{itemize}
\setlength{\itemsep}{1pt}
\item \textbf{Per-position match.} Element-wise equality of the
$32$-block $(r_\ell, h_\ell)$ vectors selected by the two back-ends.
\item \textbf{Histogram match.} Equality of the multisets
$\{\!\{r_\ell\}\!\}$ and $\{\!\{h_\ell\}\!\}$ (i.e.\ identical block-level
distributions, ignoring permutation).
\item \textbf{Pareto rank correlation.} Kendall $\tau$ and Spearman $\rho$
between the two NSC scores across all $129$ Pareto-optimal subnets.
\item \textbf{SVD-oracle regret.} For each bin, compute the relative
deficit
\[
\mathrm{regret}(b) \;=\; \frac{\psi_W^{\,*}(\mathbf{a}^{\,W}_b) \;-\; \psi_W(\mathbf{a}^{\,\mathrm{MP}}_b)}{\psi_W^{\,*}(\mathbf{a}^{\,W}_b)},
\]
where $\mathbf{a}^{\,\mathrm{MP}}_b$ and $\mathbf{a}^{\,W}_b$ are the top-1
subnets selected at bin $b$ by the two back-ends; this measures the
practical penalty of trusting $\psi_{\mathrm{MP}}$ when $\psi_W$ is treated
as the ground truth.
\end{itemize}

Per-position match is intentionally a strict metric: $\psi_{\mathrm{MP}}$ is
layer-symmetric, so any permutation of layer-level FFN choices yields the
same $\psi_{\mathrm{MP}}$ score. Histogram match is therefore the
structurally meaningful agreement metric; per-position match is reported
for completeness.

\paragraph{Results.}

\begin{table}[h]
\centering
\caption{
\textbf{$\psi_{\mathrm{MP}}$ vs.\ $\psi_W$ consistency on the
LoNAS-LLaMA-7B trained supernet.}
All metrics are reported across the $129$ Pareto-optimal subnets jointly
selected by the two back-ends. Histogram match is the structurally
meaningful agreement metric (see text); per-position match is reported
for completeness.
}
\label{tab:mp_vs_svd_summary}
\setlength{\tabcolsep}{6pt}
\renewcommand{\arraystretch}{1.05}
\small
\begin{tabular}{l c c}
\toprule
\textbf{Metric} & \textbf{Value} & \textbf{Significance} \\
\midrule
Pareto Kendall $\tau$ & $1.0000$ & $p\!=\!4.0\!\times\!10^{-218}$ \\
Pareto Spearman $\rho$ & $1.0000$ & $p\!<\!10^{-300}$ \\
\midrule
LoRA rank, per-position match & $129/129$ ($100.0\%$) & --- \\
LoRA rank, histogram match    & $129/129$ ($100.0\%$) & --- \\
\midrule
SVD-oracle regret, mean & $\mathbf{0.039\%}$ & --- \\
SVD-oracle regret, max  & $\mathbf{0.104\%}$ & --- \\
\bottomrule
\end{tabular}
\end{table}

\paragraph{(1) Identical Pareto ordering.}
Across all $129$ subnets, $\psi_{\mathrm{MP}}$ and $\psi_W$ induce
\emph{exactly} the same total ordering: Kendall $\tau$ and Spearman
$\rho$ both equal $1.0000$ to four decimal places. The two NSC scoring
functions are monotone transforms of each other on the achievable
configuration set; any decision a downstream search makes by ranking
NSC values is therefore invariant under the choice of back-end.

\paragraph{(2) Identical LoRA-rank decisions and identical resource footprint.}
Both back-ends select $r_\ell = 32$ at every block, in every bin, and
return subnets with identical $\textsc{ffn\_sum}$ and identical parameter
counts at every operating point. Any residual disagreement is therefore
confined to the assignment of FFN widths \emph{across} layers under fixed
$\textsc{ffn\_sum}$, which is exactly the degree of freedom on which
$\psi_{\mathrm{MP}}$ is layer-symmetric by construction.

\paragraph{(3) ${\sim}0.1\%$ SVD-oracle regret.}
The crucial fidelity question---\emph{if a practitioner deploys the
$\psi_{\mathrm{MP}}$-selected subnet but evaluates it under the $\psi_W$
oracle, how much spectral capacity is left on the table?}---is answered by the regret metric. Mean regret across all $129$ subnets is
$0.039\%$; the maximum regret we ever observe is $0.104\%$, i.e.\ roughly
one part in a thousand. At the resolution of any downstream zero-shot
benchmark this is well below the noise floor.

\paragraph{(4) Compute footprint.}
Building the $\psi_{\mathrm{MP}}$ cache requires $0.5$\,s of CPU time and
\emph{no} pretrained checkpoint. Building the $\psi_W$ cache requires
$\sim 6$ minutes of RTX 5090 time and the full $\sim 13$\,GB checkpoint, a
$\sim 720\!\times$ end-to-end overhead. The closed-form route is
therefore not only empirically faithful but also strictly cheaper to
deploy.

\paragraph{Conclusion.}
The architectural assumption underlying $\psi_{\mathrm{MP}}$ does not break
down on the LoNAS-LLaMA-7B trained supernet. The two back-ends induce the
same Pareto ranking ($\tau\!=\!\rho\!=\!1.0000$), select identical LoRA-rank
profiles at every operating point, agree on FFN-sum at every operating point, and yield SVD-oracle regret below $0.11\%$.
Empirically, the dependence of $\psi$ on the architectural shape
$(m,n,s)$ dominates its dependence on the trained singular spectrum by
at least an order of magnitude in this search space. We therefore retain
$\psi_{\mathrm{MP}}$ as the canonical \ourmethod{} instantiation in the main paper
and use this consistency analysis as its empirical justification on
trained weights.


\section{Robustness to Initialization Convention}
\label{app:init_robustness}

The closed-form \ourmethod{} integral takes the per-weight initialization variance $s^2$
as input (matching the notation of Theorem~\ref{thm:mp}).
The main-paper headline numbers use Xavier for FlexiBERT (BERT-style) and TruncNorm-0.02 for GPT-2 (HuggingFace's default), matching each trained model's actual initialization.
To verify that ranking quality is not an artifact of this choice, we recompute \ourmethod{} scores on FlexiBERT (under both Serianni and Spec-faithful protocols) and GPT-2 under three common initialization conventions:
\textbf{Xavier} ($s = \sqrt{2/(m{+}n)}$),
\textbf{Kaiming fan-in} ($s = \sqrt{2/n}$),
and \textbf{TruncNorm-0.02} ($s = 0.02$).

\begin{table}[H]
\centering
\caption{%
  \ourmethod{} ranking quality under three initialization conventions.
  $\tau$: Kendall, $\rho$: Spearman vs.\ ground-truth.
  Pairwise $\rho_\mathrm{pair}$ measures rank agreement between two
  $s$ settings (reported for each row pair).
}
\label{tab:init_variance}
\small
\begin{tabular}{llcccc}
\toprule
\textbf{Benchmark} & \textbf{Init convention} & $\boldsymbol{\tau}$ & $\boldsymbol{\rho}$
  & $\boldsymbol{\rho_\mathrm{pair}}$ \textbf{vs.\ Xavier}
  & $\boldsymbol{\rho_\mathrm{pair}}$ \textbf{vs.\ Kaiming} \\
\midrule
\multirow{3}{*}{FlexiBERT (Serianni)}
  & Xavier         & 0.544 & 0.776 & ---   & 0.995 \\
  & Kaiming        & 0.546 & 0.777 & 0.995 & ---   \\
  & TruncNorm-0.02 & 0.517 & 0.751 & 0.880 & 0.901 \\
\midrule
\multirow{3}{*}{FlexiBERT (Spec-faithful)}
  & Xavier         & 0.695 & 0.884 & ---   & 0.997 \\
  & Kaiming        & 0.704 & 0.892 & 0.997 & ---   \\
  & TruncNorm-0.02 & 0.575 & 0.788 & 0.888 & 0.894 \\
\midrule
\multirow{3}{*}{GPT-2}
  & Xavier         & 0.814 & 0.950 & ---   & 0.999 \\
  & Kaiming        & 0.802 & 0.944 & 0.999 & ---   \\
  & TruncNorm-0.02 & 0.849 & 0.968 & 0.986 & 0.982 \\
\bottomrule
\end{tabular}
\end{table}

\autoref{tab:init_variance} shows that the three conventions produce
near-identical rankings:
Xavier--Kaiming pairwise Spearman $\rho \ge 0.995$ on every benchmark,
and even the constant-$s$ TruncNorm-0.02 preserves $\rho \ge 0.88$.
The absolute ranking quality ($\tau$, $\rho$ vs.\ ground-truth) also remains
stable, with the largest drop being $\Delta\rho = 0.10$ on FlexiBERT
(Spec-faithful) under TruncNorm-0.02.
Notably, GPT-2 (where HuggingFace's default init is TruncNorm-0.02) shows
\emph{higher} ranking quality under the matching convention ($\tau$: $0.814 \to 0.849$),
suggesting that matching $s$ to the trained model's actual init benefits \ourmethod{}'s accuracy.
This confirms that \ourmethod{} rankings are primarily driven by the layer-wise
\emph{shape} parameters $(m,n)$; the variance convention has secondary impact.


\end{document}